\documentclass{article}
\usepackage[dvips]{graphicx}
\usepackage{amssymb,amsmath,color}
\usepackage{url}

\usepackage[mathcal]{eucal}

 \usepackage[colorlinks,
            linkcolor=blue,
            citecolor=blue,
            urlcolor=magenta,
            linktocpage,
            plainpages=false]{hyperref}

\newcommand{\BEAS}{\begin{eqnarray*}}
\newcommand{\EEAS}{\end{eqnarray*}}
\newcommand{\BEA}{\begin{eqnarray}}
\newcommand{\EEA}{\end{eqnarray}}
\newcommand{\BEQ}{\begin{equation}}
\newcommand{\EEQ}{\end{equation}}
\newcommand{\BIT}{\begin{itemize}}
\newcommand{\EIT}{\end{itemize}}
\newcommand{\BNUM}{\begin{enumerate}}
\newcommand{\ENUM}{\end{enumerate}}
\newcommand{\BA}{\begin{array}}
\newcommand{\EA}{\end{array}}

\newcommand{\rb}{\mathbb{R}}
\newcommand{\BlackBox}{\rule{1.5ex}{1.5ex}}  

\newenvironment{proof}{\par\noindent{\bf Proof\ }}{\hfill\BlackBox\\[2mm]}
\newtheorem{lemma}{Lemma}
\newtheorem{theorem}{Theorem}
\newtheorem{proposition}{Proposition}
\newtheorem{definition}{Definition}

\newcommand{\mysec}[1]{Section~\ref{sec:#1}}
\newcommand{\eq}[1]{Eq.~(\ref{eq:#1})}
\newcommand{\myfig}[1]{Figure~\ref{fig:#1}}

\def \E{{\mathbb E}}

\def \E{{\mathbb E}}

\def \X{{\mathcal X}}

\title{Submodular Functions: from Discrete to Continous Domains}

\author{Francis Bach  \\
INRIA - Sierra project-team\\
D\'epartement d'Informatique de l'Ecole Normale Sup\'erieure \\
Paris, France \\
\texttt{francis.bach@ens.fr}}

\date{\today}

\begin{document}

\maketitle

\begin{abstract}
Submodular set-functions have many applications in combinatorial optimization, as they can be minimized  and approximately maximized in polynomial time. A key element in many of the algorithms and analyses is the possibility of extending the submodular set-function to a convex function, which opens up tools from convex optimization. Submodularity goes beyond set-functions and has naturally been considered for problems with multiple labels or for functions defined on continuous domains, where it corresponds essentially to cross second-derivatives being nonpositive. In this paper, we show that most results relating submodularity and convexity for set-functions can be extended to all submodular functions. In particular, (a) we naturally define a continuous extension in a set of probability measures, (b) show that the extension is convex if and only if the original function is submodular, (c) prove that the problem of minimizing a submodular function is equivalent to a typically non-smooth convex optimization problem, and (d) propose another convex optimization problem with better computational properties (e.g., a smooth dual problem). Most of these extensions from the set-function situation are obtained by drawing links with the theory of multi-marginal optimal transport, which provides also a new interpretation of existing results for set-functions.
We then provide practical algorithms to minimize generic submodular functions on discrete domains, with associated convergence rates.
\end{abstract}

\section{Introduction}

Submodularity has emerged as an important concept in combinatorial optimization, akin to convexity in continuous optimization, with many applications in machine learning, computer vision or signal processing~\cite{boykov2001fast,krause11submodularity,lin2011-class-submod-sum,fot_submod}. Most of the literature on submodular functions focuses on \emph{set-functions}, i.e., functions defined on the set of subsets of a given base set. Such functions are classically equivalently obtained as functions defined on the vertices of the hypercube $\{0,1\}^n$, if $n$ is the cardinality of the base set. Throughout the paper, we will make this identification and refer to \emph{set-functions} as functions defined on $\{0,1\}^n$.

Like convex functions, submodular set-functions can be minimized exactly in polynomial time, either through combinatorial algorithms akin to max-flow algorithms~\cite{schrijver2000combinatorial,iwata2001combinatorial,orlin2009faster}, or algorithms based on a convex optimization techniques~\cite[Section 10]{fot_submod}. Unlike convex functions, submodular set-functions can also be \emph{maximized} approximately in polynomial time with simple greedy algorithms that come with approximation guarantees~\cite{nemhauser1978analysis,feige2007maximizing}.

In this paper, we focus primarily on submodular function minimization and links with convexity. In the set-function situation, it is known that submodular function minimization is equivalent to a convex optimization problem, which is obtained by considering a \emph{continuous extension} of the submodular function, from vertices of the hypercube $\{0,1\}^n$ to the full hypercube $[0,1]^n$. This extension, usually referred to as the Choquet integral~\cite{choquet1953theory} or the Lov\'asz extension~\cite{lovasz1982submodular}, is convex if and only if the original set-function is submodular. Moreover, when the set-function is submodular, minimizing the original set-function or the convex extension is equivalent. Finally, simple and efficient algorithms based on generic convex optimization algorithms may be used for minimization~(see, e.g.,~\cite{fot_submod}).

The main goal of this paper is to show that all of these results naturally extend to all submodular functions defined more generally on subsets of $\rb^n$.
 In this paper, we focus on functions defined on subsets of $\rb^n$ of the form $\X = \prod_{i=1}^n \X_i$, where each $\X_i   $ is a \emph{compact} subset of $\rb$. A function $H: \X \to \rb$, is then submodular if and only if 
for all $(x,y) \in \X \times \X$,  
$$H(x)+H(y) \geqslant H( \max\{x,y\}) + H(\min \{x,y\}),
$$
where the $\min$ and $\max$ operations are applied component-wise. This extended notion of submodularity has been thoroughly studied~\cite{lorentz1953inequality,topkis1978minimizing}, in particular in economics~\cite{milgrom1994monotone}.

\paragraph{Finite sets.}
Some of the results on submodular set-functions (which correspond to $\X_i = \{0,1\}$ for all $i \in \{1,\dots,n\}$) have already been extended, such as the possibility of minimizing discrete functions (i.e., when all $\X_i$'s are finite) in polynomial time. This is done usually by a reduction to the problem of minimizing a submodular function on a ring family~\cite{schrijver2000combinatorial}. Moreover, particular examples, such as certain cuts with ordered labels~\cite{ishikawa2003exact,pock2009convex,bae2014fast} lead to min-cut/max-flow reformulations with efficient algorithms. Finally, another special case corresponds to functions defined as sums of local functions, where it is known that the usual linear programming relaxations are tight for  submodular functions (see~\cite{Werner-PAMI07,zwp14:mit} and references therein). Moreover, some of these results extend to continuous Markov random fields~\cite{wald2014tightness,ruozzi}, but those results depend primarily on the decomposable structure of graphical models, while the focus of our paper is independent of decomposability of functions in several factors.

\paragraph{Infinite sets.}
While finite sets already lead to interesting applications in computer vision~\cite{ishikawa2003exact,pock2009convex,bae2014fast}, functions defined on products of sub-intervals of $\rb$ are particularly interesting. Indeed, when twice-differentiable, a function is submodular if and only if all cross-second-order derivatives are non-positive, i.e., for all $x \in \X$:
$$
\frac{\partial^2 H}{\partial x_i \partial x_j}(x) \leqslant 0.
$$
In this paper, we provide simple algorithms based solely on function evaluations to minimize all of these functions. This thus opens up a new set of ``simple'' functions that can be efficiently minimized, which neither is included nor includes convex functions, with potentially many interesting theoretical or algorithmic developments.

In this paper, we make the following contributions, all of them are extensions of the set-function case:
\BIT
\item[--] We propose in \mysec{ext} a continuous extension in a set of probability measures and  show in \mysec{extension} that it is is convex if and only if the function is submodular and  that for submodular functions, minimizing the extension, which is a convex optimization problem, is equivalent to minimizing the original function. This is made by drawing links with the theory of optimal transport, which provides simple intuitive proofs (even for submodular set-functions).

\item[--] We show in \mysec{sep} that minimizing the extension plus a well-defined separable convex function is equivalent to minimizing a series of submodular functions. This may be useful algorithmically because the resulting optimization problem is strongly convex and thus may be easier to solve using duality with Frank-Wolfe methods~\cite{bach2015duality}.

\item[--] For finite sets, we show in  \mysec{discrete} a direct link with existing notions for submodular set-functions, such as the base polytope. In the general situation, two polyhedra naturally emerge (instead of one). Moreover, the greedy algorithm to maximize linear functions on these polyhedra are also extended.

\item[--] For finite sets, we provide in \mysec{algo} two sets of algorithms for minimizing submodular functions, one based on a non-smooth optimization problem on a set of measures (projected subgradient descent), and one based on smooth functions and ``Frank-Wolfe'' methods~(see, e.g., \cite{jaggi,bach2015duality} and references therein). They can be readily applied to all situations by discretizing the sets if continuous. We provide in \mysec{experiments} a simple experiment on a one-dimensional signal denoising problem. 

\EIT
 
 In this paper, we assume basic knowledge of convex analysis (see, e.g.,~\cite{boyd,borwein2006caa}), while the relevant notions of  submodular analysis (see, e.g.,~\cite{fujishige2005submodular,fot_submod}) and optimal transport (see, e.g.,~\cite{villani2008optimal,santambrogio2015optimal}) will be rederived as needed.

\section{Submodular functions}
Throughout this paper, we consider a \emph{continuous function} $H: \X = \prod_{i=1}^n \X_i \to \rb$, defined on the product of $n$ \emph{compact} subsets $\X_i$ of $\rb$, and thus equipped with a \emph{total order}. Typically, $\X_i$ will be a finite set such as $\{0,\dots,k_i-1\}$, where the notion of continuity is vacuous, or an interval (which we refer to as a \emph{continuous domain}).

\subsection{Definition}

The function $H$ is said to be \emph{submodular} if and only if~\cite{lorentz1953inequality,topkis1978minimizing}:
\BEQ
\label{eq:submoddef}
\forall (x,y) \in \X \times \X, \ H(x)+H(y) \geqslant H( \min\{x,y\}) + H(\max \{x,y\}),
\EEQ
where the $\min$ and $\max$ operations are applied component-wise. An important aspect of submodular functions is that the results we present in this paper only rely on considering sets $\X_i$ with a \emph{total order}, from $\{0,1\}$ (where this notion is not striking) to sub-intervals of $\rb$. For submodular functions defined on more general lattices, see~\cite{fujishige2005submodular,topkis2011supermodularity}.

 Like for set-functions, an equivalent definition is that for any $x \in \X$ and (different) basis vectors $e_i,e_j$ and $a_i,a_j \in \rb_+$ such that $x_i + a_i \in \X_i$ and $x_j + a_j \in \X_j$, then
\BEQ
\label{eq:diff}
H(x+a_i e_i) + H(x+a_j e_j)  \geqslant H(x) + H(x+a_i e_i + a_j e_j).
\EEQ
Moreover, we only need the statement above in the limit of $a_i$ and $a_j$ tending to zero (but different from zero and such that
$x_i + a_i \in \X_i$ and $x_j + a_j \in \X_j$): for discrete sets included in $\mathbb{Z}$, then we only need to consider $a_i=a_j=1$, while for continuous sets, this will lead to second-order derivatives.

\paragraph{Modular functions.} We define modular functions as functions $H$ such that both $H$ and $-H$ are submodular. This happens to be equivalent to $H$ being a \emph{separable} function, that is a function which is a sum of $n$ functions that depend arbitrarily on single variables~\cite[Theorem 3.3]{topkis1978minimizing}.
These will play the role that linear functions play for convex functions. Note that when  each $\X_i$ is a sub-interval of~$\rb$, this set of functions is much larger than the set of linear functions.

\paragraph{Submodularity-preserving operations.} Like for  set-functions,  the set of submodular functions is a cone, that is, the sum of two submodular functions is submodular and multiplication by a positive scalar preserves submodularity.
Moreover, restrictions also preserve submodularity: any function defined by restriction on a product of subsets of $\X_i$ is submodular. This will be useful when discretizing a continuous domain in \mysec{algo}.

Moreover, submodularity is invariant by \emph{separable strictly increasing reparameterizations}, that is, if for all $i \in \{1,\dots,n\}$, $\varphi_i: \X_i \to \X_i$ is a strictly increasing bijection, $H$ is submodular, if and only if, $x \mapsto H\big[ \varphi_1(x_1),\dots, \varphi_n(x_n)\big]$ is submodular. Note the difference with convex functions which are invariant by \emph{affine} reparameterizations.

Finally, partial minimization does preserve submodularity, that is, if $H: \prod_{i=1}^n \X_i \to \rb$ is submodular so is
$(x_1,\dots,x_k) \mapsto \inf_{x_{k+1}, \dots, x_n} H(x)$ (exact same proof as for set-functions), while maximization does not, that is the pointwise maxima of two submodular functions may not be submodular.

\paragraph{Set of minimizers of submodular functions.} Given a submodular function, the set $\mathcal{M}$ of minimizers of $H$ is a sublattice of $\X$, that is, if if $(x,y) \in \mathcal{M} \times \mathcal{M}$, then $\max\{x,y\}$ and $\min \{x,y\}$ are also in $\mathcal{M}$~\cite{topkis1978minimizing}.

\paragraph{Strict submodularity.} We define the notion of strict submodularity through a strict inequality in \eq{submoddef} for any two $x$ and $y$ which are not comparable~\cite{topkis1978minimizing}, where, like in the rest of this paper, we consider the partial order on $\rb^n$ such that $x \leqslant x'$ if and only if $x_i \leqslant x_i'$ for all $i \in \{1,\dots,n\}$.
Similarly, this corresponds to a strict inequality in \eq{diff} as soon as both $a_i$ and $a_j$ are strictly positive. The set of minimizers of a strictly submodular function is a chain, that is the set $\mathcal{M}$ of minimizers  is totally ordered~\cite{topkis1978minimizing}.

\subsection{Examples}
\label{sec:examples}
In this section, we provide simple examples of submodular functions. We will use as running examples throughout the paper: submodular set-functions defined on $\{0,1\}^n$ (to show that our new results directly extend the ones for set-functions), modular functions (because they provide a very simple example of the concepts we introduce), and functions that are sums of terms $\varphi_{ij}(x_i-x_j)$ where $\varphi_{ij}$ is convex (for the link with Wasserstein distances between probability measures~\cite{villani2008optimal,santambrogio2015optimal}).

\BIT
\item[--] \textbf{Set-functions}: When each $\X_i$ has exactly two elements, e.g., $\X_i  = \{0,1\}$, we recover exactly submodular set-functions defined on $\{1,\dots,n\}$, with the usual identification of $\{0,1\}^n$ with the set of subsets of $\{1,\dots,n\}$. Many examples may be found in~\cite{fot_submod,fujishige2005submodular}, namely cut functions, entropies, set covers, rank functions of matroids, network flows, etc.

\item[--] \textbf{Functions on intervals}: When each $\X_i$ is a interval of $\rb$ and $H$ is twice differentiable on $\X$, then $H$ is submodular if and only if all cross-second-derivatives are non-negative, i.e., 
$$
\forall i \neq j, \forall x \in \X,\  \frac{\partial^2 H}{\partial x_i \partial x_j}(x) \leqslant 0.
$$
This can be shown by letting $a_i$ and $a_j$ tend to zero in \eq{diff}. A sufficient condition for strict submodularity is that the cross-order derivatives are strictly negative. As shown in this paper, this class of functions can be minimized efficiently while having potentially many local minima and stationary points (see an example in \myfig{example2d}).

\begin{figure}

\begin{center}
\includegraphics[scale=.5]{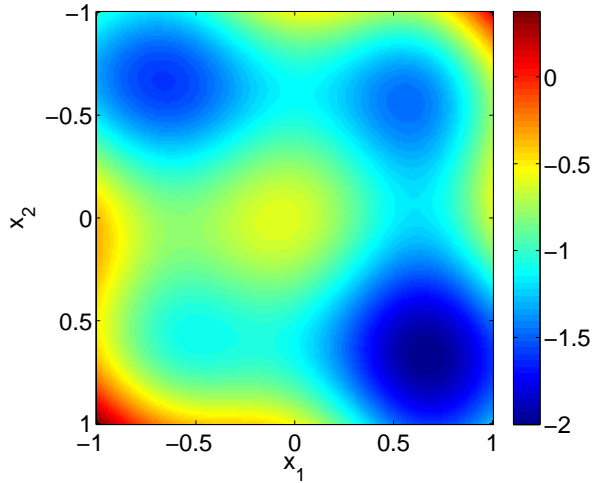}
\end{center}

\vspace*{-.6cm}

\caption{Level sets of the submodular function $(x_1,x_2) \mapsto \frac{7}{20}( x_1 - x_2)^2 - e^{-4(x_1-\frac{2}{3})^2} - \frac{3}{5}e^{ - 4 (x_1+\frac{2}{3})^2}
- e^{-4(x_2-\frac{2}{3})^2} -  e^{ - 4 (x_2+\frac{2}{3})^2}$, with  several local minima,  local maxima and   saddle points.
}
\label{fig:example2d}
\end{figure}

A quadratic function $x \mapsto x^\top  Q x$ is submodular if and only if all off-diagonal elements of $Q$ are non-positive, a class of quadratic functions with interesting behavior, e.g., tightness of semi-definite relaxations~\cite{kim2003exact}, which is another instance of the good behavior of such functions.

The class of submodular functions includes functions of the form $\varphi_{ij} (x_i -x _j)$ for $\varphi_{ij}: \rb \to \rb$  convex, and $x \mapsto g( \sum_{i=1}^n \lambda_i x_i)$ for $g$ concave and $(\lambda_i)_{i=1,\dots,n}$ non-negative weights; this gives examples of functions which are submodular, but convex or concave.

Other examples are $\varphi(\sum_{i=1}^n \lambda_i x_i) - \sum_{i=1}^n \lambda_i \varphi(x_i)$ for $\varphi$ strictly concave and $\lambda \in \rb^n$ in the interior of the simplex, which is non-negative and zero if and only if all $x_i$ are equal.

Moreover, functions of the form $x \mapsto \log \det \big( \sum_{i=1}^n x_i A_i \big)$, where $A_i$ are positive definite matrices and $x \geqslant 0$, are submodular---this extends to other spectral functions~\cite{friedland2013submodular}. Moreover, if $g$ is the Lov\'asz extension of a submodular set-function, then it is submodular (as a function defined on $\rb^n$)---see proof in Appendix~\ref{app:subsub}. These give examples of functions which are both convex and submodular. Similarly, the multi-linear extension of a submodular set-function~\cite{vondrak2008optimal}, defined on $[0,1]^n$, is submodular as soon as the original set-function is submodular (see Appendix~\ref{app:subsubmulti}), but is not convex in general.

For algorithms, these functions will be approximated on a discrete grid (one separate grid per variable, with a total complexity which is linear in the dimension $n$), but most of our formulations and convex analysis results extend in the continuous setting with appropriate regularity assumptions. 

\item[--] \textbf{Discrete labels}: in this paper, we will often consider the case where the sets $\X_i$ are all finite. They will serve as approximations of   functions defined on intervals. We will still use a functional notation to make  the extension to continuous settings explicit. Examples of functions are naturally obtained from restrictions of functions defined on continuous intervals. Moreover, as shown in \cite[Theorem 5.2]{topkis1978minimizing}, any Lipschitz-continuous submodular function defined on a product of subsets of $\rb^n$ may be extended into a Lipschitz-continuous function on $\rb^n$ (with the same constant).

\item[--] \textbf{Log-supermodular densities}: Submodular functions have also been studied as negative log-densities of probability distributions. These distributions  are referred to as ``multivariate totally positive of order 2'' and classical examples are the multivariate logistic, Gamma and $F$ distributions, as well as characteristic roots of random Wishart matrices (see more examples and additional properties in~\cite{karlin1980classes}).

\EIT

\paragraph{Submodular minimization problems.} 
In this paper, we focus on simple and efficient methods to minimize general submodular functions, based only on function evaluations. Many examples come from signal and image processing, with functions to minimize of the form 
$$H(x) = \sum_{i=1}^n f_i(x_i) + \sum_{C \in \mathcal{C}} f_C(x_C),$$
where $\mathcal{C}$ is a set of small subsets (often a set of edges) and each $f_C$ is submodular (while each $f_i$ may be arbitrary)~\cite{ishikawa2003exact,pock2009convex,bae2014fast}.
We consider a simple one-dimensional example in \mysec{experiments} as an illustration.

Another motivating example is  a probabilistic modelling problem where submodularity on continuous domains appears naturally, namely probabilistic models on $\{0,1\}^n$ with log-densities which are negatives of submodular functions~\cite{djolonga2014map,djolonga15scalable}, that is $\gamma(x) = \frac{1}{Z} \exp( - F(x) )$, with $F$ submodular and $Z$ the normalizing constant equal to $Z = \sum_{x \in \{0,1\}^n} \exp( - F(x) )$, which is typically hard to compute. In this context, \emph{mean field} inference aims at approximating $p$ by a product of independent distributions $\mu(x) = \prod_{i=1}^n \mu_i(x_i)$, by minimizing the Kullback-Leibler divergence between $\mu$ and $\gamma$, that is, by minimizing
$$
\sum_{x \in \X} \mu(x) \log \frac{\mu(x)}{\gamma(x)}
= \sum_{i=1}^n \big\{ \mu_i(1) \log \mu_i(1) + \mu_i(0) \log \mu_i(0) \big\} + \sum_{x \in \X} \mu(x) F(x) + Z.
$$
The first term in the right-hand side is separable, hence submodular, while the second term is exactly the multi-linear extension of the submodular set-function $F$, which is itself a submodular function (see~\cite{vondrak2008optimal} and Appendix~\ref{app:subsubmulti}). This implies that in this context, mean field inference may be done globally with arbitrary precision in polynomial time. Note that in this context, previous work~\cite{djolonga2014map} has considered replacing the multi-linear extension by the Lov\'asz extension, which is also submodular but also convex (it then turns out to correspond to a different divergence than the KL divergence between $\mu$ and $\gamma$).

\section{Extension to product probability measures}

Our goal is to minimize the function $H$ through a tight convex relaxation. Since all our sets $\X_i$ are subsets of $\rb$, we could look for extensions to $\rb^n$ directly such as done for certain definitions of discrete convexity~\cite{favati,fujishige2000notes}; this in fact exactly the approach for functions defined on $\{0,1\}^n$, where one defines extensions on $[0,1]^n$. The  view that we advocate in this paper is that $[0,1]$ is in bijection with the set of distributions on $\{0,1\}$ (as the probability of observing $1$).

When the sets $\X_i$ have more than two elements, we are going to consider the convex set $\mathcal{P}(\X_i)$ of \emph{Radon probability measures}~\cite{rudin1986real} $\mu_i$ on $\X_i$, which is the closure (for the weak topology) of  the convex hull of all Dirac measures; for $\X_i =\{0,\dots,k_i-1\}$, this is essentially a simplex in dimension $k_i$. In order to get an \emph{extension}, we   look for a function defined on the set of \emph{products of probability measures}
$\mu \in \mathcal{P}^\otimes(\X) = \prod_{i=1}^n \mathcal{P}(\X_i)$, such that if all $\mu_i$, $i=1,\dots,n$, are Dirac measures at points $x_i \in \X_i$, then we have a function value equal to $H(x_1,\dots,x_n)$.

We will define two types of extensions for all functions, not necessarily submodular, one based on inverse cumulative distribution functions, one based on convex closure. The two will happen to be identical for submodular functions.

\subsection{Extension based on inverse cumulative distribution functions}
\label{sec:ext}

 For a probability distribution $\mu_i \in \mathcal{P}(\X_i)$ defined on a totally ordered set $\X_i$, we can define the (reversed) cumulative distribution function $F_{\mu_i}: \X_i \to [0,1]$ as
$ F_{\mu_i}(x_i) = \mu_i \big(
\{ y_i \in \X_i, y_i \geqslant x_i \}
\big)$. This is a non-increasing left-continuous function from $\X_i$ to $[0,1]$. Such that $F_{\mu_i}( \min \X_i) = 1$ and $F_{\mu_i}(\max \X_i) = \mu_i ( \{ \max \X_i \})$. 
See  illustrations in \myfig{cumcont} and \myfig{cum} (left). 

When $\X_i$ is assumed discrete with $k_i$ elements, it may be exactly represented as a vector in $\rb^{k_i-1}$ elements with non-decreasing components, that is, given $\mu_i$, we define $ \mu_i(x_i)+ \cdots + \mu_{i}(k_i-1) = F_{\mu_i}(x_i)$, for $x_i \in \{1,\dots,k_i-1\}$. Because the measure $\mu_i$ has unit total mass,   $F_{\mu_i}(0)$ is always equal to $1$ and can thus be omitted to obtain a simpler representation (as done in \mysec{discrete}). For example, for $k_i=2$ (and $\X_i= \{0,1\}$), then we simply have $F_{\mu_i}(1) \in [0,1]$ which represents the probability that the associated random variable is equal to $1$.

Note that in order to preserve the parallel with submodular set-functions, we choose to deviate from the original definition of the cumulative function by considering the mass of the set $\{ y_i \in \X_i, y_i \geqslant x_i\}$ (and not the other direction).

We can define the ``inverse'' cumulative function from $[0,1]$ to $\X_i$ as
$$
F_{\mu_i}^{-1}(t_i) = \sup \{ x_i \in \X_i, \ F_{\mu_i}(x_i) \geqslant t_i \}.  
$$
The function $F_{\mu_i}^{-1}$ is non-increasing and right-continuous,  and such that 
$F_{\mu_i}^{-1}(1) = \min \X_i$ and $F_{\mu_i}^{-1}(0) = \max \X_i$. Moreover,  we have $F_{\mu_i}(x_i) \geqslant t_i \Leftrightarrow F_{\mu_i}^{-1}(t_i) \geqslant x_i$.   See an illustration in \myfig{cum} (right).
When $\X_i$ is assumed discrete, $F_{\mu_i}^{-1}$ is piecewise constant with steps at every $t_i$ equal to $F_{\mu_i}(x_i)$ for a certain $x_i$. For $k_i=2$, we get $F_{\mu_i}^{-1}(t_i) = 1$ if $t_i < \mu_i(1)$ and $0$ if $t_i \geqslant \mu_i(1)$. What happens at $t_i = \mu_i(1)$ does not matter because this corresponds to a set of zero Lebesgue measure.

\begin{figure}

\begin{center}
\includegraphics[scale=.6]{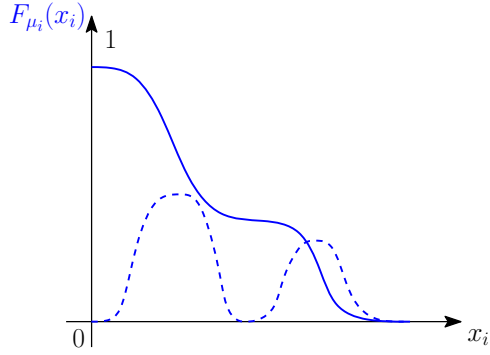}
\end{center}

\vspace*{-.4cm}

\caption{Cumulative function for a ``continuous'' distribution on the real line, with the corresponding density (with respect to the Lebesgue measure) in dotted.}
\label{fig:cumcont}
\end{figure}

\begin{figure}

\begin{center}
\includegraphics[scale=.6]{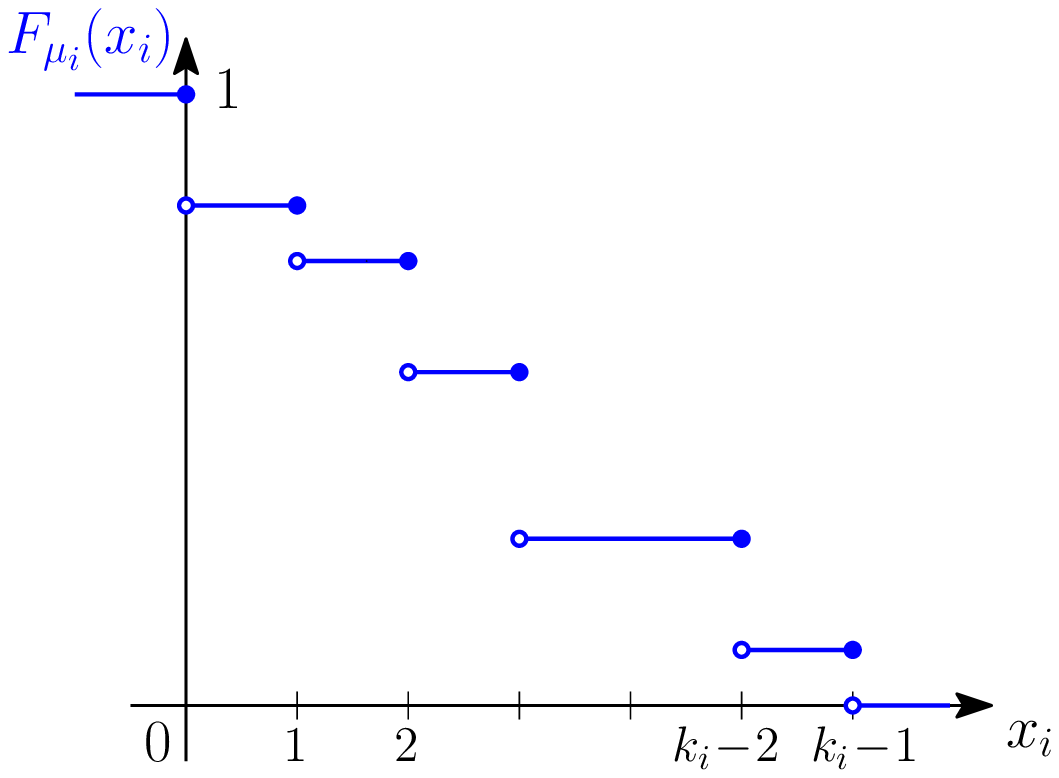}
\includegraphics[scale=.6]{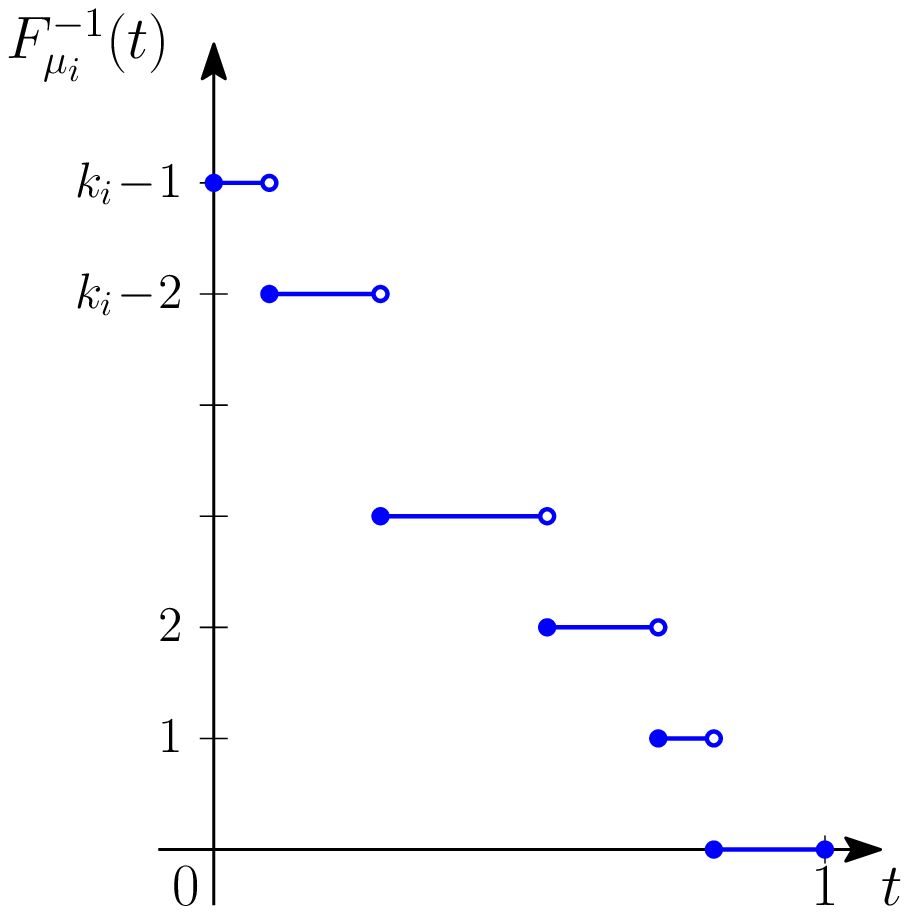}
\end{center}

\vspace*{-.4cm}

\caption{Left: cumulative function for a distribution on the real line supported in the set $\{0,\dots,k_i-1\}$. 
 Right:  inverse cumulative function (which would be the same for the distribution with discrete domain).}
\label{fig:cum}
\end{figure}

We now define our extension from $\X$ to the set of product probability measures, by considering a single threshold $t$ applied to all $n$ cumulative distribution functions. See an illustration in \myfig{transport}.

\begin{definition}[Extension based on cumulative distribution functions] Let $H: \prod_{i=1}^n \X_i \to \rb$ be any continuous function. 
We define the extension $h_{\rm cumulative}$ of $H$ to $\mathcal{P}^\otimes(\X) = \prod_{i=1}^n \mathcal{P}(\X_i)$ as follows:
\BEQ
\label{eq:ext}
\forall \mu \in \prod_{i=1}^n \mathcal{P}(\X_i), \ 
\ h_{\rm cumulative} (\mu_1,\dots,\mu_n) = \int_0^1 H\big[ F_{\mu_1}^{-1}(t),\dots, F_{\mu_n}^{-1}(t) \big] dt.
\EEQ
\end{definition}

If all $\mu_i$, $i=1,\dots,n$ are Diracs at $x_i \in \X_i$, then for all $t \in (0,1)$, $F_{\mu_i}^{-1}(t) = x_i$ and  we indeed have the extension property (again, what happens for $t=0$ or $t=1$ is irrelevant because this is on a set of zero Lebesgue measure). For $\X_i = \{0,1\}$ for all $i$, then the extension is defined on $[0,1]^n$ and is equal to
$h_{\rm cumulative}(\mu) = \int_0^1 H ( 1_{\{ \mu(1) \geqslant t\}}) dt$ and we exactly recover the Choquet integral (i.e., the Lov\'asz extension) for set-functions (see~\cite[Prop. 3.1]{fot_submod}). These properties are summarized in the following proposition.

\begin{proposition}[Properties of extension]
For any function $H: \prod_{i=1}^n \X_i \to \rb$, the extension $h_{\rm cumulative}$ satistfies the following properties:
\BIT
\item If $\mu$ is a Dirac at $x \in \X$, then $h_{\rm cumulative}(\mu) = H(x)$.
\item If all $\X_i$ are finite, then $h_{\rm cumulative}$ is piecewise affine.
\EIT

\end{proposition}

Note that the extension is defined on all tuples of measures $\mu=(\mu_1,\dots,\mu_n)$ but it can equivalently be defined through non-increasing functions from $\X_i$ to $[0,1]$, e.g., the   representation in terms of cumulative distribution functions $F_{\mu_i}$ defined above. As we will see for discrete domains in \mysec{discrete}, it may also be defined for all non-increasing functions with no contraints to be in $[0,1]$. Moreover, this extension can be easily computed, either by sampling, or, when all $\X_i$ are finite, by sorting all values of $F_{\mu_i}(x_i)$, $i \in \{1,\dots,n\}$ and $x_i \in \X_i$ (see \mysec{discrete} for details).

\paragraph{Examples.}
For our three running examples, we may look at the extension. For  set-functions, we recover the usual Choquet integral; for modular functions $H(x) = \sum_{i=1}^n H_i(x_i)$, then we have $h(\mu) = \sum_{i=1}^n \int_{\X_i} H_i(x_i) d \mu_i(x_i)$ which is the expectation of $H(x)$ under the product measure defined by $\mu$. Finally, for the function $\varphi_{ij}(x_i-x_j)$, we obtain a Wasserstein distance between the measures $\mu_i$ and $\mu_j$ (which is a distance between their cumulative functions)~\cite{villani2008optimal}. See more details in \mysec{transport}.

\subsection{Extension based on convex closures}
\label{sec:extclos}
\label{sec:closure}

We now describe a second way of extending a function defined on $\X$ to a function defined on $\mathcal{P}^\otimes(\X) = \prod_{i=1}^n \mathcal{P}(\X_i)$, using the concept of \emph{convex closure}. We consider the function $g$ defined on 
$\mathcal{P}^\otimes(\X) $ as $g(\mu) = H(x)$ if $\mu$ is the Dirac measure $\delta_x$ at $x \in \X$, and $+\infty$ otherwise. The following proposition gives an expression for its convex closure, in terms of  Kantorovich multi-marginal optimal transport~\cite{carlier2003class,villani2008optimal}, which looks for a joint probability measure on $\X$ with given marginals on each $\X_i$, $i=1,\dots,n$.

\begin{proposition}[Extension by convex closure - duality]
 \label{prop:closure}
 Assume all $\X_i$ are compact subsets of~$\rb$, for $i \in \{1,\dots,n\}$ and that $H: \X \to \rb$ is continuous.
 The largest lower-semi continuous (for the weak topology on Radon measures) convex function $h : \prod_{i=1}^n \mathcal{P}(\X_i) \to \rb$ such that $h(\delta_x) \leqslant H(x)$ for any $x \in \X$ is equal to
\BEQ
\label{eq:closure}
h_{\rm closure}(\mu_1,\dots,\mu_n) = \inf_{ \gamma  \in \mathcal{P}(\X) } \int_{\X} H(x) d\gamma(x)  ,
\EEQ
where the infimum is taken
 over all probability measures $\gamma$ on $\X$ such that the $i$-th marginal $\gamma_i$ is equal to $\mu_i$. Moreover, the infimum is attained and we have the dual representation:
 \BEQ
\label{eq:dual}
h_{\rm closure}(\mu)   =  \sup_v \sum_{i=1}^n \int_{\X_i} v_i(x_i) d \mu_i(x_i) \mbox{ such that }
\forall x \in \X, \ \sum_{i=1}^n v_i(x_i) \leqslant H(x_1,\dots,x_n),
\EEQ
over all continuous functions $v_i: \X_i \to \rb$, $i=1,\dots,n$.
We denote by $\mathcal{V}(H)$ the set of such potentials $v = (v_1,\dots,v_n)$.
 \end{proposition}
 \begin{proof}
The function $h_{\rm closure}$ can be computed as the Fenchel bi-conjugate of $g$. The first step is to compute $g^\ast(v)$ for $v$ in the dual space to $\prod_{i=1}^n \mathcal{P}(\X_i)$, that is $v \in \prod_{i=1}^n \mathcal{C}(\X_i)$, with $\mathcal{C}(\X_i)$ the set of continuous functions on $\X_i$. We have, by definition of the Fenchel-Legendre dual function, with
$\langle \mu_i, v_i \rangle = \int_{\X_i} v_i(x_i) d\mu_i(x_i)$   the integral of $v_i$ with respect to $\mu_i$:

$$
g^\ast(v)  =  \sup_{ \mu \in  \mathcal{P}^\otimes(\X)} \sum_{i=1}^n \langle \mu_i, v_i \rangle - g(\mu) = \sup_{ x \in \X}  \ \sum_{i=1}^n v_i(x_i)  - H(x) .
$$
This supremum is equal to
\BEQ
\label{eq:dualopt}
g^\ast(v)  = \sup_{ \gamma \in  {\mathcal{P}(\X)} }
 \int_\X  \big\{  \sum_{i=1}^n v_i(x_i)  - H(x) \big\}d\gamma(x)
 \EEQ
over all probability measures $\gamma$ on $\X$. We may then expand using $\gamma_i$ the $i$-th marginal of $\gamma$ on $
\X_i$ defined as $\gamma_i(A_i) = \gamma( \{ x \in \X, \ x_i \in A_i\})$ for any measurable set $A_i \subset \rb$, as follows:
$$
g^\ast(v)   =  \sup_{ \gamma \in \mathcal{P}(\X)}  \bigg\{ \sum_{i=1}^n \int_{\X_i} v_i(x_i)  {d \gamma_i}(x_i)  -  
 \int_\X   H(x) d \gamma(x) \bigg\} .
$$
The second step is to compute the bi-dual $g^{\ast\ast}(\mu) = \sup_{ v\in \prod_{i=1}^n \mathcal{C}(\X_i)  } \sum_{i=1}^n \langle \mu_i,v_i \rangle - g^\ast(v)$ for $\mu \in  \prod_{i=1}^n \mathcal{P}(\mathcal{\X}_i)$:
\BEAS
g^{\ast\ast}(\mu)  & = &  \sup_{ v \in \prod_{i=1}^n \mathcal{C}(\X_i)  } \sum_{i=1}^n \langle v_i , \mu_i \rangle
-  \sup_{ \gamma \in \mathcal{P}(\X)}  \bigg\{ \sum_{i=1}^n \sum_{x_i \in \X_i} v_i(x_i)  {\gamma_i}(x_i)  -  \int_\X   H(x) d \gamma(x)\bigg\}  \\
& = &   \sup_{ v\in \prod_{i=1}^n \mathcal{C}(\X_i) }   \inf_{ \gamma \in \mathcal{P}(\X)} 
  \sum_{i=1}^n \int_{\X_i}  v_i(x_i) \big( {d\gamma_i}(x_i)   - d\mu_i(x_i) \big) -   \int_\X   H(x) d \gamma(x) \\
& = &   \inf_{ \gamma \in \mathcal{P}(\X)}  \sup_{ w\in \prod_{i=1}^n \mathcal{C}(\X_i) }  
  \sum_{i=1}^n \int_{\X_i}  v_i(x_i) \big( {d\gamma_i}(x_i)   - d\mu_i(x_i) \big) -   \int_\X   H(x) d \gamma(x).
\EEAS
In the last equality, we use strong duality which holds here because of the continuity of $H$ and the compactness of all sets $\X_i$, $i=1,\dots,n$. See for example~\cite{villani2008optimal} for details. Note that the infimum in $\gamma$ is always attained in this type of optimal transport problems.

 Thus, by maximizing over each $v_i \in \mathcal{C}(\X_i)$, we get and additional constraint and thus $  \displaystyle g^{\ast\ast}(\mu) = \inf_{ \gamma  \in \mathcal{P}(\X) } \int_{\X} H(x) d\gamma(x)$ such that $\forall i, \ \gamma_i  = \mu_i $. This leads to the desired result.
 \end{proof}
 The extension by convex closure $h_{\rm closure}$ has several interesting properties, independently of the submodularity of $H$, as we now show.
 
 \begin{proposition}[Properties of convex closure]
 \label{prop:propclo}
 For any continuous function  $H: \prod_{i=1}^n \X_i \to \rb$, the extension $h_{\rm closure}$ satistfies the following properties:
\BIT
\item[(a)] If $\mu$ is a Dirac at $x \in \X$, then $h_{\rm closure}(\mu) \leqslant H(x)$.
\item[(b)] The function $h_{\rm closure}$ is convex.
\item[(c)] Minimizing $h_{\rm closure}$ on $\prod_{i=1}^n \mathcal{P}(\X_i)$ and minimizing $H$ on $\prod_{i=1}^n \X_i$ is equivalent, that is, the two optimal values are equal, and one may find minimizers of one problem given the other one.
\EIT
 \end{proposition}
\begin{proof}
Property (a) is obvious from the definition. Note that in general, the inequality may be strict (it will not for submodular functions). Since the objective function and constraint set in \eq{closure} are jointly convex in $\gamma$ and $\mu$, the infimum with respect to $\gamma$ is thus convex in $\mu$, which implies property (b). In order to show (c), we note that $\inf_{\mu \in \mathcal{P}^\otimes(\X)} h_{\rm closure}(\mu)$ is trivially less than $\inf_{x \in \X} H(x)$ because of (a), and we consider the sequence of equalities:
\BEAS
\inf_{x \in \X} H(x)& = &  \inf_{\gamma \in \mathcal{P}(\X)} \int_{\X} H(x) d \gamma(x)
 = \inf_{\mu \in \mathcal{P}^\otimes(\X)}  \inf_{\gamma \in \mathcal{P}(\X)}  \int_{\X} H(x) d \gamma(x)
 \mbox{ such that } \forall i, \ \gamma_i = \mu_i \\
 & = & \inf_{\mu \in \mathcal{P}^\otimes(\X)}  h(\mu).
\EEAS
Moreover, any optimal $\gamma$ is supported on the (compact) set of minimizers of $H$ on $\X$. Thus any optimal $\mu$ is the set of marginals of any distribution $\gamma$ supported on the minimizers of $H$.
\end{proof} 
 
While the convex closure is attractive because it is convex and allows the minimization of $H$,  the key difficulty in general is that $ {h}_{\rm closure}$ cannot be computed in general. These are opposite properties to the extension $h_{\rm cumulative}$ based on cumulative distribution functions. We now show that the two extensions are equal when $H$ is submodular.

\subsection{Equivalence between the two extensions through one-dimensional optimal transport}
\label{sec:transport}
We have seen two possible extensions of $H: \X \to \rb$ to $h:  \mathcal{P}^\otimes(\X) \to \rb$. When $H$ is submodular, the two are equal, as a consequence of the following proposition, which is itself obtained directly from the theory of multi-marginal optimal transport between one-dimensional distributions~\cite{carlier2003class}.

\begin{proposition}[One-dimensional multi-marginal transport]
\label{prop:transport}
Let $\X_i$ be a compact subset of~$\rb$, for $i \in \{1,\dots,n\}$ and $H$ a continuous submodular function defined on
$\X = \prod_{i=1}^n \X_i$. Then the two functions $h_{\rm cumulative}$ defined in \eq{ext} and $h_{\rm closure}$ defined in \eq{closure} are equal.

\end{proposition}

\begin{proof}
Since~\cite{carlier2003class} does not provide exactly our result, we give a detailed proof here from first principles.
In order to prove this equivalence, there are three potential proofs: (a) based on convex duality, by exhibiting primal and dual candidates (this is exactly the traditional proof for submodular set-functions~\cite{lovasz1982submodular,edmonds}, which is cumbersome for continuous domains, but which we will follow in \mysec{discrete} for finite sets); (b) based on Hardy-Littlewood's inequalities~\cite{lorentz1953inequality}; or (c) using properties of optimal transport. We consider the third approach (based on the two-marginal proof of~\cite{santambrogio2015optimal}) and use four steps:
\BIT
\item[(1)] We define $\gamma_{\rm mon} \in \mathcal{P}(\X)$ as the distribution of $(F_{\mu_1}^{-1}(t),\dots, F_{\mu_n}^{-1}(t) ) \in \X$ for $t$ uniform in $[0,1]$. See an illustration in \myfig{transport}.
The extension $h_{\rm cumulative}$ corresponds to the distribution $\gamma_{\rm mon}$ and we thus need to show that $\gamma_{\rm mon}$ is an optimal distribution. In this context, probability distributions, i.e., elements of $\mathcal{P}(\X)$ with given marginals are often referred to as ``transport plan'', a terminology we now use.

 The transport plan $\gamma_{\rm mon}$ is trivially ``monotone'' so that two elements of its support are comparable for the partial order $x \leqslant x'$ if all components of $x$ are less than or equal to the corresponding components of $x'$. Moreover, is it such that
$\gamma_{\rm mon} \Big(
 \prod_{i=1}^n \big\{ y_i \in \X_i,  y_i \geqslant x_i \big\}
 \Big) $
 is the Lebesgue measure of the set of $t \in [0,1]$ such that $F_{\mu_i}^{-1}(t) \geqslant x_i$ for all $i$, that is such that 
 $F_{\mu_i}(x_i) \leqslant t$ for all $i$, thus 
 \BEQ
 \label{eq:gammmon}
 \gamma_{\rm mon} \bigg(
 \prod_{i=1}^n \big\{ y_i \in \X_i,  y_i \geqslant x_i \big\}
 \bigg) = \max_{i \in \{1,\dots,n\}} F_{\mu_i}(x_i).
 \EEQ

\begin{figure}

\begin{center}
\includegraphics[scale=.6]{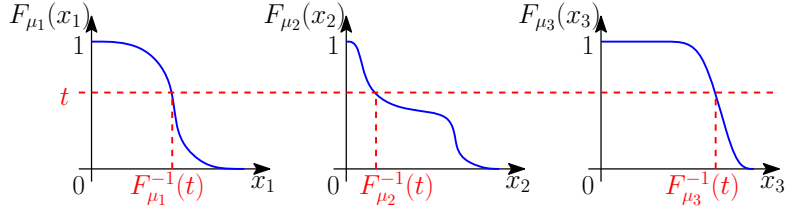}
\end{center}

\vspace*{-.5cm}

\caption{Multi-marginal optimal transport by thresholding inverse cumulative distribution functions: definition of the transport plan $\gamma_{\rm mon}$.}
\label{fig:transport}
\end{figure}

\item[(2)] We show that if $\gamma \in \mathcal{P}(\X)$ is a distribution so that  any two elements $x,x' \in \X$ of its support are comparable, then it is equal to $\gamma_{\rm mon}$. We simply need to compute the mass of a product of rectangle as in \eq{gammmon}. For $n=2$ marginals, we consider the 4 possible combinations of the sets $\big\{ y_i \in \X_i,  y_i \geqslant x_i \big\}$, $i \in \{1,2\}$ and their complements. because of the comparability assumption, either $\big\{ y_1 \in \X_1,  y_1 \geqslant x_1 \big\} \times \big\{ y_2 \in \X_2,  y_2 < x_2 \big\} $ or 
$\big\{ y_1 \in \X_1,  y_1 < x_1 \big\} \times \big\{ y_2 \in \X_2,  y_2 \geqslant x_2 \big\} $ is empty (see \myfig{monotone}), which implies that the measure of $\big\{ y_1 \in \X_1,  y_1 \geqslant x_1\big\} \times \big\{ y_2 \in \X_2,  y_2\geqslant x_2 \big\}$ is either
$F_{\mu_1}(x_1)$ or $F_{\mu_2}(x_2)$, and hence larger than the maximum of these two. The fact that it is lower is trivial, hence the result.

\begin{figure}

\begin{center}
\includegraphics[scale=.6]{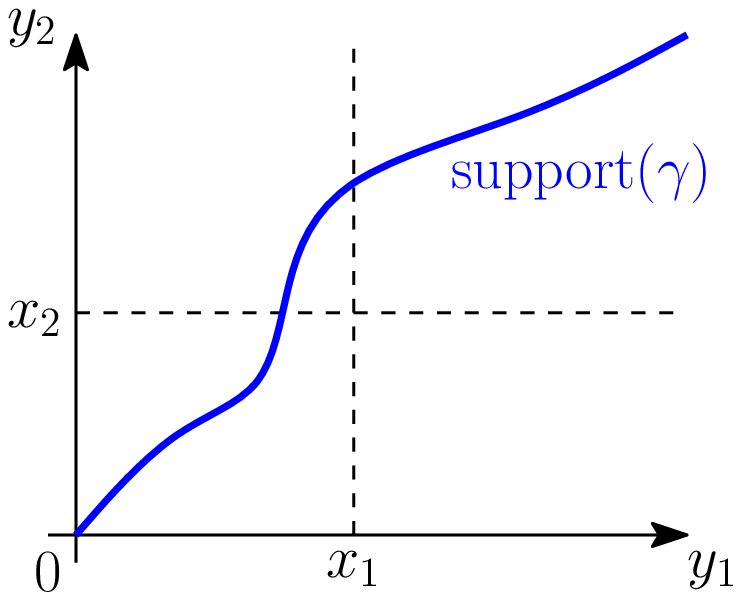}
\end{center}

\vspace*{-.5cm}

\caption{Monotone transport plan $\gamma$ such that $\big\{ y_1 \in \X_1,  y_1 \geqslant x_1 \big\} \times \big\{ y_2 \in \X_2,  y_2 < x_2 \big\} $  is empty, leading to $\gamma\Big(\big\{ y_1 \in \X_1,  y_1 \geqslant x_1\big\} \times \big\{ y_2 \in \X_2,  y_2\geqslant x_2 \big\}\Big) = F_{\mu_1}(x_1)$.}
\label{fig:monotone}
\end{figure}

\item[(3)]  If $H$ is strictly submodular, any optimal transport plan $\gamma \in \mathcal{P}(\X)$ satisfies the property above of monotone support. Indeed, let us assume that $x$ and $x'$ are two non-comparable elements of the support of $\gamma$. From convex duality used in the proof of Proposition~\ref{prop:closure}, in particular, \eq{dualopt}, there exist continuous potentials $v_i: \X_i \to \rb$ such that $H(x) = \sum_{i=1}^n v_i(x_i)$ and $H(x') = \sum_{i=1}^n v_i(x_i')$ (because of complementary slackness applied to any element of the support of $\gamma$), while for any $y \in \X$, we simply have $H(y) \geqslant \sum_{i=1}^n v_i(y_i)$. By considering $y = \max\{x,x'\}$ and $y = \min\{x,x'\}$, we obtain:
$
H(x) + H(x') \leqslant H( \max\{x,x'\})
+  H( \min\{x,x'\})
$, which is in contradiction with the strict submodularity of $H$. Thus any optimal plan has to be equal to $\gamma_{\rm mon}$ for strictly submodular functions.

\item[(4)] When $H$ is submodular, by adding $\varepsilon \sum_{ i \neq j} (x_i - x_j)^2$, we obtain a strictly submodular function and by letting $\varepsilon$ tend to zero, we obtain the desired result.

\EIT
\end{proof}

\subsection{Relationship between convexity and submodularity}

\label{sec:extension}  
We can now prove our first formal result relating convexity and submodularity, that extends the similar result of~\cite{lovasz1982submodular} from set-functions to all continuous functions. Given the theory of multi-marginal optimal transport outlined above, the proof is straightforward and our result provides an alternative proof even for set-functions; note that while an implicit connection had been made for $n=2$ through monotonicity properties of optimal assignment problems~\cite{shapley1962complements}, the link we propose here is novel.
\begin{theorem}[Convexity and submodularity]
\label{theo:lov}
 Assume $H$ is continuous and all $\X_i$'s are compact. The extension $h_{\rm cumulative}$ defined in \eq{ext} is convex if and only if $H$ is submodular.
\end{theorem}
\begin{proof}
  We first assume that $H$ is submodular.
As shown in Proposition~\ref{prop:transport}, optimal transport problems on subsets of real numbers with submodular costs are known to have closed-form solutions~\cite{carlier2003class}, which leads to the convexity of $h_{\rm cumulative}
= h_{\rm closure}$.

We now assume that the function $h_{\rm cumulative}$ is convex. Following the argument of~\cite{lorentz1953inequality} for the related problem of rearrangement inequalities and~\cite{lovasz1982submodular} for submodular set-functions, we consider two arbitrary elements $a$ and $b$ in $\X$, and the Dirac measures $\delta_{a_i}$ and $\delta_{b_i}$. We have, by convexity:
$$h_{\rm cumulative}\big( \frac{1}{2} \delta_{a_1} +\frac{1}{2}\delta_{b_1}, \dots, \frac{1}{2} \delta_{a_n} +\frac{1}{2}\delta_{b_n}\big)
\leqslant \frac{1}{2} h_{\rm cumulative} (  \delta_{a_1} , \dots,  \delta_{a_n}   )
+\frac{1}{2} h_{\rm cumulative} (  \delta_{b_1} , \dots,  \delta_{b_n}   ).
$$
The right-hand side is equal to $\frac{1}{2} H(a) + \frac{1}{2} H(b)$, while we can compute the left-hand side by computing
$
F_{\frac{1}{2} \delta_{a_i} +\frac{1}{2}\delta_{b_i}}(x_i)
$, which is equal to $0$ if $x_i > \max \{a_i,b_i\}$, $1$ if $x_i < \min \{a_i,b_i\}$ and $\frac{1}{2}$ if $x_i 
\in ( \min \{a_i,b_i\}, \max \{a_i,b_i\} )$. This implies that $
F_{\frac{1}{2} \delta_{a_i} +\frac{1}{2}\delta_{b_i}}^{-1}(t)$ is equal to $ \min \{ a_i, b_i\}$ if $t>\frac{1}{2}$, and to 
$ \max \{ a_i, b_i\}$ if $t<\frac{1}{2}$. Thus the left-hand side of the inequality above is equal to
$\frac{1}{2} H(\min \{a,b\} ) + \frac{1}{2} H(\max \{a,b\})$. Hence, the submodularity.
\end{proof}

From now on, we will assume that $H$ is submodular and refer to $h$ as its extension, which is both defined as a convex closure and through cumulative distribution functions. Note that a consequence of Theorem~\ref{theo:lov} is that for submodular functions, the closure of the sum is the sum of the closures, which is  not true in general.

We now show that minimizing the extension is equivalent to minimizing  the original function, implying that we may minimize any submodular function as a convex optimization problem:
\begin{theorem}[Equivalent minimization problems]
\label{theo:eq}
 Assume each $\X_i$ is compact, $i=1,\dots,n$.
If $H$ is submodular, then 
\BEQ
\displaystyle \inf_{x \in \X} H(x) = \inf_{ \mu \in \mathcal{P}^\otimes(\X)} h(\mu).
\EEQ
Moreover, $\mu$ is a minimizer if and only if $F_\mu^{-1}(t)$ is a minimizer of $H$ for almost all $t \in [0,1]$. 
\end{theorem}
\begin{proof}
Since $h$ is the convex closure of $H$, the two infima have to be equal. Indeed, from Proposition~\ref{prop:propclo}, we have
$$
\inf_{ \mu \in \mathcal{P}^\otimes(\X)} h(\mu)
= \inf_{ \mu \in \mathcal{P}^\otimes(\X)}  \inf_{ \gamma \in \Pi(\mu) } \int_{\X} H(x) d\gamma(x),
$$
which is the infimum over all probability measures on $\X$ (without any marginal constraints). It is thus achieved at a Dirac measure at any minimizer $x \in \X$ of $H(x)$.

Moreover,  given a minimizer $\mu$ for the convex problem, we have:
$$\inf_{x \in \X} H(x) = h(\mu) =   \int_0^1 H \big[  F_{\mu_1}^{-1}(t),\dots, F_{\mu_n}^{-1}(t) \big] dt  \geqslant \int_0^1\inf_{x \in \X} H(x) dt = \inf_{x \in \X} H(x).$$
 Thus, for almost all $t \in [0,1]$, $\big( F_{\mu_1}^{-1}(t),\dots, F_{\mu_n}^{-1}(t)  \big) \in \X$ is a minimizer of $H$.
\end{proof}
From the proof above, we see that a minimizer of $H$ may be obtained from a minimizer of $h$ by inverting the cumulative distribution for any $t$. Many properties are known regarding minimizers of submodular functions~\cite{topkis1978minimizing}, i.e., if $x$ and $y$ are minimizers of $H$, so are $\min\{x,y\}$ and $\max\{x,y\}$. As opposed to convex function where imposing strict convexity leads to a unique minimizer, imposing strict submodularity only imposes that   the set of minimizers forms a chain.   

In practice, given a (potentially approximate) minimizer $\mu$ and for discrete domains, we can look at the minimal value of $ H \big[  F_{\mu_1}^{-1}(t),\dots, F_{\mu_n}^{-1}(t) \big]$ over all $t \in [0,1]$ by enumerating and sorting all possible  values of $F_{\mu_i}(x_i)$ (see \mysec{greedy}). Moreover, we still need a subgradient of $h$ for our optimization algorithms. We will consider these in the simpler situation of finite sets in \mysec{discrete}.

\paragraph{Dual of submodular function minimization.}
Algorithms for minimizing submodular set-functions rely on a dual problem which allows to provide optimality certificates. We obtain a similar dual problem in the general situation as we now show:

\begin{proposition}[Dual of submodular function minimization]
\label{prop:dualsfm}
The problem of minimizing $H(x)$ over $x \in \X$, has the following dual  formulation 
\BEQ
\label{eq:dualsfm}
\inf_{x \in \X} H(x) = \inf_{ \mu \in \mathcal{P}^\otimes(\X)} h(\mu) 
=  \sup_{v\in \mathcal{V}(H)} \sum_{i=1}^n \inf_{x_i \in \X_i} v_i(x_i).
\EEQ
\end{proposition}
\begin{proof}
We have $\inf_{x \in \X} H(x) = \inf_{ \mu \in \mathcal{P}^\otimes(\X)} h(\mu) $ from Theorem~\ref{theo:eq}. Moreover, we may use convex duality~\cite{mitter2008convex} like in Prop.~\ref{prop:closure} to get:
\BEAS
\inf_{ \mu \in \mathcal{P}^\otimes(\X)} h(\mu)  & = & 
\inf_{ \mu \in \mathcal{P}^\otimes(\X)}\sup_{v \in \mathcal{V}(H)}  
 \sum_{i=1}^n  \int_{\X_i}v_i(x_i) d\mu_i(x_i) \\
 & = & 
\sup_{v \in \mathcal{V}(H)}  \sum_{i=1}^n \inf_{\mu_i \in   \mathcal{P}(\X_i)} \int_{\X_i}v_i(x_i) d\mu_i(x_i)
=  \sup_{v\in \mathcal{V}(H)} \sum_{i=1}^n \inf_{x_i \in \X_i} v_i(x_i).
\EEAS
\end{proof}
It allows to provide certificates of optimality when minimizing $H$. Note that like for set-functions, checking that a given function $v$ is in $\mathcal{V}(H)$ is difficult, and therefore, most algorithms will rely on convex combinations of outputs of the greedy algorithm presented in \mysec{discrete}.

 \subsection{Strongly-convex separable submodular function minimization}
 \label{sec:sep}
 \label{sec:separable}
 
 In the set-function situation, minimizing the Lov\'asz extension plus a separable convex function has appeared useful in several scenarios~\cite{chambolle2009total,fot_submod}, in particular because a single of such problems leads to a solution to a continuously parameterized set of submodular minimization problems on~$\X$. In our general formulation, the separable convex functions that can be combined are themselves defined through a submodular function and optimal transport.
 
 We choose strongly convex separable costs of the form $\sum_{i=1}^n \varphi_i(\mu_i)$, with:
 \BEQ
 \label{eq:phi}
\varphi_i(\mu_i) = \int_{\X_i} a_i(x_i, F_{\mu_i}(x_i)) dx_i,
\EEQ
where for all $x_i \in \X_i$, $a_i(x_i,\cdot)$ is a differentiable $\lambda$-strongly convex function on $[0,1]$. This implies that the function $\varphi_i$, as a  function of $F_{\mu_i}$ is $\lambda$-strongly convex (for the $L_2$-norm on cumulative distribution functions). A key property is that it may be expressed as a transport cost, as we now show.

\begin{proposition}[Convex functions of cumulative distributions]
\label{prop:cumuc}
For $a_i: \X_i \times [0,1] \to \rb$ convex and differentiable with respect to the second variable, the function $\varphi_i: \mathcal{P}(\X_i) \to \rb$ defined in \eq{phi} is equal to
 \BEQ
 \label{eq:phic}
\varphi_i(\mu_i) = \int_0^1 c_i( F_{\mu_i}^{-1} (t),1-t) dt,
\EEQ
for $\displaystyle c_i(z_i,t_i) =  \int_{\X_i} a_i(x_i,0) dx_i + 
 \int_{\X_i} \bigg(  \int_{\X_i \cap (-\infty, z_i]  }   \frac{\partial a_i}{\partial t_i}(x_i, 1- t_i)  
dx_i  \bigg)   $ is a submodular cost on $\X_i \times [0,1]$.
\end{proposition}
 \begin{proof}
 We have the following sequence of equalities:
 \BEAS
 \varphi_i(\mu_i) & = &  \int_{\X_i} a_i(x_i, F_{\mu_i}(x_i)) dx_i \\
 & = &  \int_{\X_i} \bigg( a_i(x_i,0) + 
 \int_{0}^{F_{\mu_i}(x_i)} \frac{\partial a_i}{\partial t_i}(x_i,t_i) dt_i  
\bigg)dx_i \\
& = & \int_{\X_i} a_i(x_i,0) dx_i + 
 \int_{\X_i} \bigg(  \int_{\X_i \cap [x_i,+\infty) }   \frac{\partial a_i}{\partial t_i}(x_i, F_{\mu_i}(y_i)) d\mu_i(y_i) 
\bigg)dx_i   \\
& & \hspace*{1cm} \mbox{ by the change of variable } t_i = F_{\mu_i}(y_i),
\\
& = & \int_{\X_i} a_i(x_i,0) dx_i + 
 \int_{\X_i} \bigg(  \int_{\X_i \cap (-\infty, y_i]  }   \frac{\partial a_i}{\partial t_i}(x_i, F_{\mu_i}(y_i))  
dx_i  \bigg)  d\mu_i(y_i) \mbox{ by Fubini's theorem} ,
\\
& = & 
 \int_{\X_i} c_i(y_i, 1 - F_{\mu_i}(y_i))  
 d\mu_i(y_i) \mbox{ by definition of } c_i ,
\\
& = & 
 \int_{0}^1 c_i( F_{\mu_i}^{-1}(t),1-t)  
 dt   \mbox{ by the change of variable } t_i = F_{\mu_i}(y_i).
 \EEAS
 Moreover, $c_i$ is indeed a submodular function as, for any $z_i' > z_i$ in $\X_i$, we have
 $c_i(z_i',t_i) - c_i(z_i,t_i) = \displaystyle 
 \int_{\X_i} \bigg(  \int_{\X_i \cap (z_i, z_i']  }   \frac{\partial a_i}{\partial t_i}(x_i, 1- t_i)  
dx_i  \bigg)   $, which is a decreasing function in $t_i$ because $a_i$ is convex. Thus $c_i$ is submodular. It is also strictly submodular if $a_i(x_i,\cdot)$ is strongly convex for all $x_i \in \X_i$.
 \end{proof}
 
Because $c_i$ is submodular, we have, following Prop.~\ref{prop:transport}, a formulation of $\varphi_i$ as an optimal transport problem between the measure $\mu_i$ on $\X_i$ and the uniform distribution $U[0,1]$ on $[0,1]$, as 
$$ \varphi_i(\mu_i) = \inf_{\gamma_i \in \mathcal{P}(\X_i \times [0,1] } \int_{\X_i \times [0,1]}  c_i(x_i,t_i) d \gamma_i(x_i,t_i),$$
such that $\gamma_i$ has marginals $\mu_i$ and the uniform distribution on $[0,1]$.
It is thus always convex (as as the minimum of a jointly convex problem in $\gamma_i$ and $\mu_i$)---note that it is already convex from the definition in \eq{phic} and the relationship between $a_i$ and $c_i$ in Prop.~\ref{prop:cumuc}.

We now consider the following problem:
\BEQ
\label{eq:prox}
\inf_{\mu \in \mathcal{P}^\otimes(\X) } h(\mu) + \sum_{i=1}^n \varphi_i(\mu_i) ,
\EEQ
which is an optimization problem with $h$, with additional separable transport costs $\varphi_i(\mu_i)$. Given our assumption regarding the strong-convexity of the functions $a_i$ above, this system has a unique solution.
We may derive a dual problem using the representation
from \eq{dual}:
\BEA
\nonumber \inf_{\mu \in \mathcal{P}^\otimes(\X)  } h(\mu) + \sum_{i=1}^n  \varphi_i(\mu_i)
& = &  \inf_{\mu \in \mathcal{P}^\otimes(\X)  } \sup_{v \in \mathcal{V}(H)} \sum_{i=1}^n \bigg\{ \int_{\X_i} v_i(x_i) d \mu_i(x_i)  +  \varphi_i(\mu_i) \bigg\}  \\
\nonumber & = &  \sup_{v \in \mathcal{V}(H)} \sum_{i=1}^n \inf_{\mu_i \in \mathcal{P}(\X_i)  }  \bigg\{ \int_{\X_i} v_i(x_i) d \mu_i(x_i)  +  \varphi_i(\mu_i) \bigg\} 
\\
\label{eq:dualprox} & = &  \sup_{v \in \mathcal{V}(H)} - \sum_{i=1}^n\varphi_i^\ast( -v_i)  ,
\EEA
where we use the Fenchel-dual notation $\varphi_i^\ast( v_i) = \sup_{\mu_i \in \mathcal{P}(\X_i)  }  \bigg\{ \int_{\X_i} v_i(x_i) d \mu_i(x_i)  -  \varphi_i(\mu_i) \bigg\}$.
The equation above provides a dual problem to \eq{prox}.
We may also consider a family of submodular minimization problems, parameterized by $t \in [0,1]$:
\BEQ
\label{eq:sfmt}
\min_{x \in \X} \ H(x) + \sum_{i=1}^n  c_i( x_i,1- t),
\EEQ
with their duals defined from \eq{dualsfm}. Note that the two dual problems are defined on the same set~$\mathcal{V}(H)$. We can now prove the following theorem relating the two optimization problems.

\begin{theorem}[Separable optimization - general case]
\label{theo:sepc}
Assume $H$ is continuous and submodular and all $c_i$, $i=1,\dots,n$ are defined as in Prop.~\ref{prop:cumuc}. Then:
\BIT
\item[(a)] If $x$ and $x'$ are minimizers of \eq{sfmt} for $t> t'$, then $x \leqslant x'$ (for the partial order on $\rb^n$), i.e., the solutions of \eq{sfmt} are non-increasing in $t \in [0,1]$.
\item[(b)] Given a primal candidate $\mu \in \mathcal{P}^\otimes(\X)$ and a dual candidate $v \in \mathcal{V}(H)$, then the duality gap for the problem in \eq{prox} is the integral from $t=0$ to $t=1$ of the gaps for the problem in \eq{sfmt} for the same dual candidate and the primal candidate $(F_{\mu_1}^{-1}(t),\dots, F_{\mu_n}^{-1}(t)) \in \X$.
\item[(c)] Given the unique solution $\mu$ of \eq{prox}, for all $t \in [0,1]$, $(F_{\mu_1}^{-1}(t),\dots, F_{\mu_n}^{-1}(t)) \in \X$ is a solution of \eq{sfmt}. 
\item[(d)] Given any solutions $x^t \in \X$ for all problems in \eq{sfmt}, we may define $\mu$ through
$F_{\mu_i}(x_i) = \sup \big\{
t \in [0,1], \ x_i^t \geqslant x_i
\big\}$, for all $i$ and $x_i \in \X_i$, so that $\mu$ is the optimal solution of \eq{prox}.
\EIT
\end{theorem}
\begin{proof}
The first statement (a) is a direct and classical consequence of the   submodularity of $c_i$~\cite[Section 2.8]{topkis2011supermodularity}. The main idea is that when we go from $t$ to $t'<t$, then the function difference, i.e.,  $x_i \mapsto  c_i( x_i, t') -  c_i( x_i, t)$ is strictly increasing, hence the minimizer has to decrease.  

For the second statement (b), we may first re-write the cost function in \eq{prox} as an integral in $t$, that is, for any $\mu \in \mathcal{P}^\otimes(\X)$:
\BEAS
h(\mu ) + \sum_{i=1}^n  \varphi_i(\mu_i) 
& = & \int_0^1
\bigg\{
H\big[ F_{\mu_1}^{-1}(t),\dots, F_{\mu_n}^{-1}(t) \big]+ \sum_{i=1}^n  c_i\big[ F_{\mu_i}^{-1} (t), 1-t \big]
\bigg\} dt.
\EEAS

The gap defined in Prop.~\ref{prop:dualsfm} for a single submodular minimization problem in \eq{sfmt} is, for a primal candidate $x \in \X$ and a dual candidate $v \in \mathcal{V}(H)$:
\BEAS
& &H(x) + \sum_{i=1}^n  c_i\big[  x_i, F_{\mu_i}^{-1} (t) \big]
- \sum_{i=1}^n  \min_{y_i \in \X_i} \Big\{  v_i(y_i)  +  c_i \big[ y_i, 1-t\big]\Big\},
\EEAS
and its integral with respect to $t \in [0,1]$ for $x_i = F_{\mu_i}^{-1}(t)$, for all $i \in \{1,\dots,n\}$, is  equal to
\BEAS
& &h(\mu ) + \sum_{i=1}^n  \varphi_i(\mu_i) 
- \sum_{i=1}^n  \int_0^1 \min_{y_i \in \X_i} \Big\{  v_i(y_i)  +  c_i \big[ y_i, 1-t \big]\Big\}  dt .
\EEAS
Finally, we have, by using the formulation of $\varphi_i(\mu_i)$ through optimal transport, an expression of the elements appearing in the dual problem in \eq{dualprox}:
\BEAS
 - \varphi_i^\ast(-v_i)& =  &   \inf_{\mu_i \in \mathcal{P}_i(\X_i)  }  \bigg\{ \int_{\X_i} v_i(x_i) d \mu_i(x_i)  +  \varphi_i(\mu_i) \bigg\} 
\\
& = &  \inf_{\mu_i \in \mathcal{P}_i(\X_i)  }   \inf_{\gamma_i \in \Pi(\mu_i,U[0,1]) } \int _{\X_i \times [0,1]} \big[ v_i(x_i) + c_i(x_i,1-t_i)  \big] d \gamma_i(x_i,t_i) 
\\
& = &  \int_0^1  \inf_{y_i \in \X_i} \big[ v_i(y_i) + c_i(y_i,1-t_i)  \big] d t_i,
\EEAS
because we may for any $t_i$ choose the conditional distribution of $x_i$ given $t_i$ equal to a Dirac at the minimizer $y_i$ of  $v_i(y_i) + c_i(y_i,1-t_i)$. This implies (b) and thus, by continuity, (c). The statement (d) is proved exactly like for set-functions~\cite[Prop.~8.3]{fot_submod}.
\end{proof}

In \mysec{base}, we will consider a formulation for finite sets that will exactly recover the set-function case, with the additional concept of base polytopes.

\section{Discrete sets}
\label{sec:discrete}
\label{sec:finite}

In this section, we consider only finite sets for all $i \in \{1,\dots,n\}$, i.e., $\X_i = \{0,\dots,k_i-1\}$. For this important subcase, we will extend many of the notions related to submodular set-functions, such as the base polytope and the greedy algorithm to compute its support function. This requires to extend the domain where we compute our extensions, from product of probability measures to products of non-increasing functions. Throughout this section, all measures are characterized by their probability mass functions, thus replacing integrals by sums.

This extension will be done using a specific representation of the measures $\mu_i \in \mathcal{P}(\X_i)$. Indeed, we may represent $\mu_i$ through its cumulative distribution function
$\rho_i(x_i) = \mu_i(x_i)+ \cdots + \mu_{i}(k_i-1) = F_{\mu_i}(x_i)$, for $x_i \in \{1,\dots,k_i-1\}$. Because the measure $\mu_i$ has unit total mass,   $\rho_i(0)$ is always equal to $1$ and can be left out. The only constraint on $\rho_i$ is that is has non-increasing components and that all of them belong to $[0,1]$. We denote by $[0,1]^{k_i-1}_\downarrow$ this set, which is in bijection with $\mathcal{P}( \{0,\dots,k_i-1\})$. Therefore, $\rho_i$ may be seen as a truncated cumulative distribution function equal to a truncation of~$F_{\mu_i}$; however,  we will next extend its domain and remove the restriction of being between 0 and 1; hence the link with the cumulative distribution function $F_{\mu_i}$ is not direct anymore, hence a new notation $\rho_i$.

We are going to consider the set of non-increasing vectors $\rb^{k_i-1}_\downarrow$ (without the constraint that they are between $0$ and $1$). For any such $\rho_i$ with $k_i-1$ non-increasing components, the set of real numbers is divided into $k_i$ parts, as shown below. Note that this is simply a rephrasing of the definition of $F_{\mu_i}^{-1}(t)$, as, when $\rho_i \in [0,1]^{k_i-1}_\downarrow$, we have $\theta(\rho_i,t) = F_{\mu_i}^{-1}(t)$.
 
\begin{center}
\includegraphics[scale=.6]{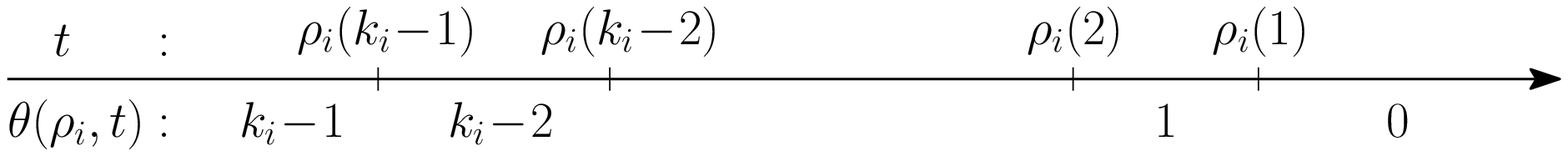}
\end{center}

This creates a map $\theta(\rho_i,\cdot): \rb \to \{0,\dots,k_i-1\}$
 such that $\theta(\rho_i,t) = k_i-1$ for $t< \rho_i(k_i-1)$, $\theta(\rho_i,t) = x_i$
  if $ t\in ( \rho_i(x_i+1), \rho_i(x_i))$, for $x_i \in \{1,\dots,k_i-2\}$, and $\theta(\rho_i,t) =0$ for $ t>\rho_i(1)$. What happens at the boundary points is arbitrary and irrelevant.

For example, for $k_i=2$ (and $\X_i= \{0,1\}$), then we simply have $\rho_i \in  \rb$ and $\theta(\rho_i,t) = 1$ for $t<\rho_i$, and $0$ if $t>\rho_i$. We can now give an expression of $h(\mu)$ as a function of $\rho \in \prod_{i=1}^n [0,1]^{k_i-1}_\downarrow$, which will be extended to all non-increasing vectors (not constrained to be between $0$ and $1$). This will then allow us to define extensions of  base polytopes.

\subsection{Extended extension on all products of non-increasing sequences}

\label{sec:greedy}
We first start by a simple lemma providing an expression of $h(\mu)$ as a function of $\rho$---note the similarity with the Lov\'asz extension for set-functions~\cite[Prop. 3.1]{fot_submod}.

\begin{lemma}[Extension as a function of $\rho$]
\label{lemma:hrho}
For $\rho \in \prod_{i=1}^n[0,1]^{k_i-1}_\downarrow$, define $h_\downarrow(\rho)$ as
\BEQ
\label{eq:hnu01}
h_\downarrow(\rho) = \int_{0}^{1} H(\theta(\rho_1,t),\dots,\theta(\rho_n,t)) dt.
\EEQ
If $\mu \in \mathcal{P}^\otimes(\X)$ and $\rho$ are linked through $\rho_i(x_i) = F_{\mu_i}(x_i)$ for all $i$ and $x_i \in \{1,\dots,k_i-1\}$, then, $h(\mu) = h_\downarrow(\rho)$.
\end{lemma}
\begin{proof}
This is simply a re-writing of the definition of $\theta$, as for almost all $t \in [0,1]$, and all $i \in \{1,\dots,n\}$, 
we have $\theta(\rho_i,t) = F_{\mu_i}^{-1}(t)$.
\end{proof}
We can give an alternative formulation for $h_\downarrow(\rho)$ in \eq{hnu01} for $\rho \in \prod_{i=1}^n[0,1]^{k_i-1}_\downarrow$, as (with $\max\{\rho\}$ the maximum value of all $\rho_i(x_i)$, and similarly for $\min\{\rho\}$), using
that $\theta(\rho_i,t) = \max \X_i = k_i-1$ for $t < \min \{\rho\}$ and 
$\theta(\rho_i,t) = \min \X_i = 0$ for $t > \max \{\rho\}$:
\BEA
\nonumber
\!\!\!\!\!\!\!\!\!\!\! h_\downarrow(\rho) & \!\!\!\!\!=\!\!\!\!\! &  \bigg( \int_{0}^{\min\{\rho\}}
+
\int_{\min\{\rho\}}^{\max\{\rho\}} 
+
\int_{\max\{\rho\}}^{1}
\bigg)
 H(\theta(\rho_1,t),\dots,\theta(\rho_n,t)) dt
\\
\label{eq:hrho}
& \!\!\! \!\!=\!\!\!\!\!   &
\int_{\min\{\rho\}}^{\max\{\rho\}} H(\theta(\rho_1,t),\dots,\theta(\rho_n,t)) dt
+ \min\{\rho\} H(k_1\!-\!1,\dots,k_n\!-\!1) + ( 1 - \max\{\rho\}) H(0). \hspace*{.5cm}
\EEA
The expression in \eq{hrho} may be used as a definition of $h_\downarrow(\rho)$ for all $\rho \in \prod_{i=1}^n \rb^{k_i-1}_\downarrow$.
With this definition, the function $\rho \mapsto h_\downarrow(\rho) - H(0)$ is piecewise linear. The following proposition shows that it is convex.

\begin{proposition}[Convexity of extended extension]
Assume $H$ is submodular. The function $h_\downarrow$ defined in \eq{hrho} is convex on $\prod_{i=1}^n \rb^{k_i-1}_\downarrow$.
\end{proposition}
\begin{proof}
From the definition in \eq{hrho}, we have that, with $\rho + C$ being defined as adding the constant $C$ to all components of all $\rho_i$'s:
$h_\downarrow(\rho + C) = h_\downarrow(\rho) + C \big[  H(k_1\!-\!1,\dots,k_n\!-\!1) - H(0) \big]$. Moreover, 
 $\rho \mapsto h_\downarrow(\rho) - H(0)$ is positively homogeneous. Thus, any set of $\rho$'s in $\prod_{i=1}^n \rb^{k_i-1}_\downarrow$ may be transformed linearly to $\prod_{i=1}^n [0,1]^{k_i-1}_\downarrow$ by subtracting the global minimal value and normalizing by the global range of all $\rho's$. Since $h_\downarrow(\rho)$ coincides with the convex function $h(\mu)$ where $\mu_i$ is the probability distribution associated (in a linear way) to $\rho_i$, we obtain the desired result.
 \end{proof}

\paragraph{Greedy algorithm.}
We now provide a simple algorithm to compute $h_\downarrow(\rho)$ in \eq{hrho}, that extends the greedy algorithm for submodular set-functions.
We thus now assume that we are given $\rho \in \prod_{i=1}^n \rb^{k_i-1}_\downarrow$, and we compute $h_\downarrow(\rho)$ without sampling. 

We first order all $r = \sum_{i=1}^n k_i - n$ values of $\rho_i(x_i)$ for all $x_i \in \{1,\dots,k_i-1\}$, in decreasing order, breaking ties randomly, except to ensure that all values for $\rho_i(x_i)$ for a given $i$ are in the correct order (such ties may occur when one associated $\mu_i(x_i)$ is equal to zero).  See \myfig{greedy} for an example.

\begin{figure}

\begin{center}
\includegraphics[scale=.6]{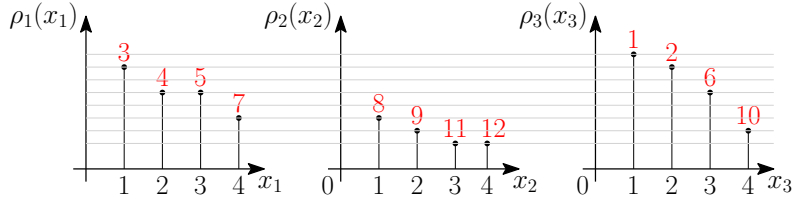}
\end{center}

\vspace*{-.5cm}

\caption{Examples of joint ordering (in red)  of all values of $\rho_i(x_i)$ for the greedy algorithm ($n=3$ and $k_i = 5$ for all $i$).}
\label{fig:greedy}
\end{figure}

We assume that the $s$-th value is equal to $t(s)$ and corresponds to $\rho_{i(s)}(j(s))$. We have $t(1) = \max\{\rho\}$ and $t(r) = \min \{\rho\}$. For $s \in \{1,\dots,r-1\}$, we define the vector $y(s) \in \X$ so that $y(s)_i$ will be the value of $\theta(\rho_i,t) $ on the open interval (potentially empty) $(t(s+1), t(s))$; note that what happens at break points is still irrelevant. By convention, we 
define $y(0) = (0,\dots,0)$ and $y(r) = (k_1\!-\!1,\dots,k_n\!-\!1)$. We then go from $y(s-1)$ to $y(s)$ by increasing the $i(s)$-th component by one. Note that because we have assumed that $(\rho_i(x_i))_{x_i \in \X_i}$ are well-ordered, we always have $y(s)_{i(s)} = j(s)$ and $y(s) = y(s-1) + e_{i(s)}$, where $e_i \in \rb^n$ is the $i$-the canonical basis vector.

 We thus get, by cutting the integral from \eq{hrho} on $[\min\{\rho\},\max\{\rho\}] = [t(r),t(1)]$ into pieces:
\BEAS
h_\downarrow(\rho) & = &  \min\{\rho\} H(k_1\!-\!1,\dots,k_n\!-\!1) + ( 1 - \max\{\rho\}) H(0) 
+ \sum_{s=1}^{r-1}  \int_{t(s+1)}^{t(s)} H\big[ y(s)\big]   dt  \\
& = & H(0) +  t(r) H\big[ y(r)\big] - t(1)H\big[ y(0)\big] + \sum_{s=1}^{r-1}\big[ t(s) - t(s+1) \big] \cdot H\big[ y(s)\big] \\
& = & H(0) + \sum_{s=1}^r t(s) \Big( 
H\big[ y(s)\big] - H\big[ y(s-1)\big]
\Big).
\EEAS
Since we have ordered all $r = \sum_{i=1}^n k_i - n$ values of $t(s) = \rho_{i(s)}(j(s))$, each $ \rho_i(x_i)$ appears exactly once in the sum above. 
Therefore, $h_\downarrow(\rho) $ is of the form $h_\downarrow(\rho) = H(0) + \sum_{i=1}^n \sum_{x_i = 1}^{k_i-1} \rho_i(x_i) w_i(x_i)$, where $w_i(x_i)$ is a difference of two function values
of $H$ at arguments that differ from a single canonical basis vector. Moreover, we have $\sum_{i=1}^n \sum_{x_i=1}^{k_i-1} w_i(x_i) = H[y(r)] - H(0) = H(k_1\!-\!1,\dots,k_n\!-\!1) - H(0)$.  We refer to this $w$ as the output of the greedy algorithm (associated with a specific ordering of the values of $\rho$).

Note that for any $\rho$, several orderings may be chosen, all leading to the same value $h_\downarrow(\rho)$ but different values for $w_i(x_i)$. The number of such ordering is $n!$ for set-functions (i.e., $k_i=2$ for all $i$) and is in general equal to
the multinomial coefficient $  \frac{ \big( \sum_{i=1}^n (k_i - 1) \big)!}{ \prod_{i=1}^n (k_i-1)!} $. If all $k_i$, $i=1,\dots,n$, are equal to $k$, then we get $  \frac{  ( n(k-1)  )!}{ (k_i-1)!^n} $, which grows very rapidly when $n$ and $k$ grow.

Note that for $k_i=2$ for all $i \in \{1,\dots,n\}$, then we exactly obtain the greedy algorithm for submodular set-functions. The complexity is $O( r \log r)$ for sorting and $r$ function evaluations, with $r = \sum_{i=1}^n k_i - n$.

\subsection{Base polyhedron}
\label{sec:base}
We may now provide an extension of the base polyhedron which is usually defined for submodular set-functions. 
As opposed to set-functions, there are two natural polyhedra $\mathcal{W}(H)$ and $\mathcal{B}(H)$, one defined by linear inequalities, one defined as the convex hull of the outputs of the greedy algorithm~\cite{fujishige2005submodular,fot_submod}. They are equal for set-functions, but not in the general case. The key difference is the monotonicity constraint on each $\rho_i$, which is only active when $k_i>2$.

We consider the set $\mathcal{W}(H) \subset \prod_{i=1}^n \rb^{k_i-1}$ defined through the inequalities:
 \BEAS
 \forall (x_1,\dots,x_n) \in \X, & &  \sum_{i=1}^n \sum_{y_i = 1}^{x_i} w_i(y_i) \leqslant H(x_1,\dots,x_n) - H(0) \\
& &  \sum_{i=1}^n \sum_{y_i = 1}^{k_i-1} w_i(y_i) = H(k_1\!-\!1,\dots,k_n\!-\!1) - H(0).
\EEAS
Note that when some $x_i=0$, then the sum $\sum_{y_i = 1}^{x_i} w_i(y_i)$ is equal to zero by convention (alternatively, we can use the convention that $w_i(0)=0$).  When $k_i=2$ for all $i \in \{1,\dots,n\}$, we obtain exactly the usual base polyhedron (which happens to be bounded)~\cite{fujishige2005submodular}. The following proposition shows that it behaves in a similar way (except that it is not bounded).

\begin{proposition}[Support function for $\mathcal{W}(H)$]
\label{prop:WH}
Assume $H$ is submodular and $\X_i = \{0,\dots,k_i-1\}$ for $i \in \{1,\dots,n\}$. Then, for any $\rho \in \prod_{i=1}^n \rb^{k_i-1}$, 
$$\sup_{w \in \mathcal{W}(H)} \sum_{i=1}^n \sum_{x_i=1}^{k_i-1} w_i(x_i) \rho_i(x_i)$$ is equal to $+\infty$ if $\rho \notin  \prod_{i=1}^n \rb^{k_i-1}_\downarrow$, and to $h_\downarrow(\rho) -H(0)$ otherwise, with an optimal $w \in \mathcal{W}(H)$ obtained from the greedy algorithm for any compatible order.
\end{proposition}
\begin{proof}
In this proof, we are going to use a reparameterization of $w$ as follows;
 we  define a vector $v_i \in \rb^{\{1,\dots,k_i-1\}}$ as the cumulative sum of $w_i$, that is, such that
 $v_i(x_i) =  \sum_{ y_i=1}^{x_i} w_i(y_i)$ (this is a bijection from $w_i$ to $v_i$). We then have the constraint that $\sum_{i=1}^n   v_i(x_i)     \leqslant H(x_1,\dots,x_n) -H(0) $, for all $x \in \X$,  with the convention that $v_i(0)=0$. The extra constraint is that $\sum_{i=1}^n v_i(k_i\!-\!1) = H(k_1\!-\!1,\dots,k_n\!-\!1)-H(0)$.
 Moreover, we have, with $\mu_i(x_i) =   \rho_i(x_i) - \rho_i(x_i\!+\!1) $ for $x_i \in \{1,\dots,k_i-2\}$ and $\mu_i(k_i-1) = \rho_i(k_i-1)$, an expression for the linear function we aim to maximize, that is, by Abel's summation formula:
 \BEAS
 \sum_{x_i=1}^{k_i-1} w_i(x_i) \rho_i(x_i) 
 & = & \sum_{x_i=2}^{k_i-1} \big[ v_i(x_i) - v_i(x_i\!-\!1) \big] \rho_i(x_i) + v_i(1) \rho_i(1) \\
 & = & \sum_{x_i=1}^{k_i-2} v_i(x_i) \big[\rho_i(x_i) - \rho_i(x_i\!+\!1) \big] + v_i(k_i\!-\!1) \rho_i(k_1 \!-\! 1) =  \sum_{x_i=1}^{k_i-1} v_i(x_i)  \mu_i(x_i).
 \EEAS

We first assume that each $\rho_i$ is non-decreasing.
We are going to follow the same proof than for set-functions, based on convex duality.   First, given the piecewise-linearity of $h_\downarrow - H(0)$ (and hence the homogeneity), and the fact that $h_\downarrow(\rho + C) = h_\downarrow(\rho) + C \big[ H(k_1-1,\dots,k_n-1) - H(0) \big] $, we only need to show the result for $\rho_i(x_i) \in [0,1]$ for all $i$, and $x_i \in \X_i$. Each vector $\rho_i$ is then uniquely associated to a probability measure (with non-negative values because $\rho_i$ is non-increasing) $\mu_i$ on $\X_i = \{0,\dots, k_i-1\}$, and, from Lemma~\ref{lemma:hrho}, $h_\downarrow(\rho) = h(\mu)$.

 Using the parameterization in terms of $v$ and $\mu$, we can now consider the Lagrangian, with dual values $\gamma \in \rb_+^\X$:
 $$
 \mathcal{L}(v,\gamma) = \sum_{i=1}^n \sum_{x_i=1}^{k_i-1} v_i(x_i)  \mu_i(x_i)
 + \sum_{x \in \X} \gamma(x) \bigg( H(x) -H(0) - \sum_{i=1}^n v_i(x_i) \delta(x_i>0)\bigg).
 $$
 By maximizing with respect to the primal variable $v$, the dual problem thus becomes the one of minimizing $\sum_{x \in \X} \gamma(x) \big[H (x)- H(0) \big]$ such that 
 \BEAS
  \forall i \in \{1,\dots,n\}, \ \forall x_i \in \{1,\dots,k_i-1\}, & &  \gamma_i(x_i) = \mu_i(x_i) \\
  \forall x \in \X \backslash \{(k_1\!-\!1,\dots,k_n\!-\!1)\}, & & \gamma(x) \geqslant 0,
 \EEAS
 where $\gamma_i$ is the marginal of $\gamma$ on the $i$-th variable, that is, $\gamma_i(x_i) = \sum_{ x_j  \in \X_j, j \neq i} \gamma(x_1,\dots,x_n)$.
 
  We now exhibit primal/dual pairs with equal objective values. For the dual variable $\gamma$, we consider the solution of the usual optimal transport problem (which happens to satisfy extra constraints, that is, nonnegativity also for $x=0$ and summing to one), and the dual value is exactly the extension $h(\mu)- H(0)$. For the primal variable, we consider the $w$ parameterization (which is in bijection with~$v$). From the greedy algorithm described in~\mysec{greedy}, the vector $w$ satisfies the sum constraint $\sum_{i=1}^n \sum_{y_i = 1}^{k_i-1} w_i(y_i) = H(k_1\!-\!1,\dots,k_n\!-\!1) -H(0)$. We simply need to show that for all $x \in \X$,
 $ \sum_{i=1}^n \sum_{z_i = 1}^{x_i} w_i(z_i) \leqslant H(x) - H(0)$. We have, using the notation of the greedy algorithm from \mysec{greedy}:
 \BEAS
 \sum_{i=1}^n \sum_{z_i = 1}^{x_i} w_i(z_i) 
 & = & \sum_{s=1}^r \big( H[y(s)] - H[y(s-1)] \big) \delta(y(s)_{i(s)} \leqslant x_{i(s)} ) \mbox{ by definition of } i(s) \mbox{ and } y(s),\\
 & = & \sum_{a=1}^n \sum_{s, i(s) = a}  \big( H[y(s)] - H[y(s-1)] \big) \delta(y(s)_{a} \leqslant x_{a} ) \mbox{ by splitting the values of } i(s),\\
 & = & \sum_{a=1}^n \sum_{s, i(s) = a}  \big( H[y(s-1) + e_a] - H[y(s-1)] \big) \delta(y(s)_{a}  \leqslant x_{a} ) \\
  & & \mbox{ because we go from } y(s-1) \mbox{ to } y(s) \mbox{ by incrementing the component } i(s)=a, \\
 & \leqslant & \sum_{a=1}^n \sum_{s, i(s) = a}  \big( H[\min\{y(s-1) + e_a,x\} ] - H[\min \{y(s-1), x\} ] \big) \delta(y(s-1)_a  + 1 \leqslant x_{a} ) \\
 & & \hspace*{.5cm} \mbox{ by submodularity and because } y(s)_{i(s)} = y(s-1)_{i(s)} + 1,\\
  & =  & \sum_{a=1}^n \sum_{s, i(s) = a}  \big( H[\min\{y(s-1) + e_a,x\} ] - H[\min \{y(s-1), x\} ] \big) \\
  & & \hspace*{.5cm} \mbox{ because the difference in values of } H \mbox{ is equal to zero for } y(s-1)_a  + 1 > x_{a}, \\
   & =  & \sum_{s=1}^r \big( H[\min\{y(s) ,x\} ] - H[\min \{y(s-1), x\} ] \big) \\ 
  & = & H[\min\{x,y(r)\}] - H[\min\{x,y(0)\}] = H(x)-H(0) .
 \EEAS
 Thus, $w$ is feasible. By construction the primal value is equal to $h(\mu)-H(0)$. We thus have a primal/dual optimal pair and the result is proved for non-decreasing $\rho_i$'s. This notably shows that the polyhedron $\mathcal{W}(H)$ is non-empty.
 
We can now show that if one $\rho_i$ is not non-decreasing, then the supremum is equal to infinity. In such a case, then there exists $x_i \in \{1,\dots,k_i-2\}$ such that $\mu_i(x_i) < 0 $. We may then let the corresponding $v_i(x_i)$ tend to $-\infty$.
 \end{proof}

Given the representation of $h_\downarrow(\rho)-H(0)$ as a maximum of linear functions (with an infinite value if some $\rho_i$ 
does not have non-increasing components), the convex problem of minimizing $H(x)$, which is equivalent to minimizing $h_\downarrow(\rho)-H(0)$ with respect to
$ \rho \in \prod_{i=1}^n [0,1]^{k_i-1}_{\downarrow}   $ has the following dual problem:
\BEA
\nonumber \min_{ \rho \in \prod_{i=1}^n [0,1]^{k_i-1}_{\downarrow}} 
h_\downarrow(\rho)-H(0)
& = & 
 \min_{ \rho \in \prod_{i=1}^n [0,1]^{k_i-1}}   \max_{w \in \mathcal{W}(H)}  \sum_{i=1}^n \sum_{x_i = 1}^{k_i-1} w_i(x_i) \rho_i(x_i)  \\
\nonumber & = & 
\max_{w \in \mathcal{W}(H)}  \sum_{i=1}^n  \min_{\rho_i \in [0,1]^{k_i-1} } \sum_{x_i = 1}^{k_i-1} w_i(x_i) \rho_i(x_i)  \\
\label{eq:dualW}
 & = & 
\max_{w \in \mathcal{W}(H)} 
\sum_{i=1}^n   \sum_{x_i=1}^{k_i-1} \min\{ w_i(x_i), 0\}.
\EEA
While $\mathcal{W}(H)$ share some properties with the base polytope of a submodular set-function, it is not bounded in general and is not the convex hull of all outputs of the greedy algorithm from \mysec{greedy}.

We now define the base polytope $\mathcal{B}(H)$ as the convex hull of all outputs of the greedy algorithm, when going over all allowed orderings of $\sum_{i=1}^n k_i -n $ elements of $\rho$ that respect the ordering of the individual $\rho_i$'s. In the submodular set-function case, we have $\mathcal{B}(H) =\mathcal{W}(H)$. However, in the general case we only have an inclusion.

\begin{proposition}[Properties of the base polytope $\mathcal{B}(H)$]
Assume $H$ is submodular. Then:
\BIT
\item[(a)] $\mathcal{B}(H) \subset \mathcal{W}(H)$.
\item[(b)] For any $\displaystyle\rho \in \prod_{i=1}^n \rb^{k_i-1}_\downarrow$, $\displaystyle\max_{ w \in \mathcal{B}(H) } \langle w, \rho \rangle = \max_{w \in \mathcal{W}(H)} \langle w, \rho \rangle = h_\downarrow(\rho)-H(0)$, with a joint maximizer obtained as the output of the greedy algorithm with any particular order.
\item[(c)] $\displaystyle\mathcal{W}(H) = \mathcal{B}(H) + \mathcal{K}$ with
$$\!\!\!
\mathcal{K} = \bigg\{ w \in \prod_{i=1}^n \rb^{k_i-1}, \ \forall i \in \{1,\dots,n\}, \ \forall x_i \in \{1,\dots,k_i-2\}, \ \sum_{y_i=1}^{x_i} w_i(y_i) \leqslant 0 \mbox{ and } \sum_{y_i=1}^{k_i-1} w_i(y_i) = 0 \bigg\}.
$$
\EIT
\end{proposition}
\begin{proof}
In the statements above, $\langle w, \rho \rangle$ stands for $\sum_{i=1}^n \sum_{x_i=1}^{k_i-1} w_i(x_i) \rho_i(x_i)$.
The  statements (a) and (b)  were shown in the proof of Prop.~\ref{prop:WH}. Statement (c) is a simple consequence of the fact that the polar cone to $\prod_{i=1}^n \rb^{k_i-1}_\downarrow$ is what needs to be added to $\mathcal{B}(H)$ to get to $\mathcal{W}(H)$. In other words, (c) is the equality of two convex sets and is thus equivalent to the equality of their support functions; and we have, for any $\rho_i \in   \rb^{k_i-1}$, using Abel's summation formula:
$$\langle w_i, \rho_i \rangle =  \sum_{x_i=1}^{k_i-1} w_i(x_i) \rho_i(x_i) 
 = \sum_{x_i=1}^{k_i-2} \bigg( \sum_{y_i=1}^{x_i} w_i(y_i) \bigg) \big[\rho_i(x_i) - \rho_i(x_i\!+\!1) \big] +
 \bigg( \sum_{y_i=1}^{x_i-1} w_i(y_i) \bigg) \rho_i(k_1 \!-\! 1), $$
 from which we see that the supremum of $\langle w, \rho \rangle $ with respect to $\rho \in  \prod_{i=1}^n  \rb^{k_i-1}_\downarrow$ is equal to zero
  if $  w \in \mathcal{K}$ and $+\infty$ otherwise. By convex duality, this implies that the supremum of 
  $\langle w, \rho \rangle$ with respect to $w \in \mathcal{K}$ is equal to zero if $\rho$ has non-increasing components and zero otherwise. We thus have, for any $\rho \in \prod_{i=1}^n \rb^{k_i-1}$,
  $\sup_{ w \in \mathcal{B}(H) + \mathcal{K} } \langle w, \rho \rangle
  = h_\downarrow(\rho)$ if $\rho \in \prod_{i=1}^n \rb^{k_i-1}_\downarrow$ and $+\infty$ otherwise. Given Prop.~\ref{prop:WH}, this leads to the desired result.
\end{proof}
The key difference between $\mathcal{W}(H)$ and $\mathcal{B}(H)$ is that the support function of $\mathcal{B}(H)$ does not include the constraint that the argument should be composed of non-increasing vectors. However, $\mathcal{B}(H)$ is a polytope (i.e., bounded as a convex hull of finitely many points), while $\mathcal{W}(H)$ is not.

To obtain dual problems, we may either choose $\mathcal{W}(H)$ or $\mathcal{B}(H)$ (and then take into account the monotonicity contraints explicitly). For our algorithms in \mysec{algo}, we will consider $\mathcal{B}(H)$, while in the section below, we will use $\mathcal{W}(H)$ for the analysis and comparison of several optimization problems. The dual using $\mathcal{W}(H)$ is given in \eq{dualW}. When using $\mathcal{B}(H)$, the dual problem of minimizing
$h_\downarrow(\rho)-H(0)$ with respect to
$ \rho \in \prod_{i=1}^n [0,1]^{k_i-1}_{\downarrow}   $ has the following form:
\BEA
& & \nonumber\min_{ \rho \in \prod_{i=1}^n [0,1]^{k_i-1}_{\downarrow} } h_\downarrow(\rho)-H(0) \\
\nonumber & =  & \min_{ \rho \in \prod_{i=1}^n [0,1]^{k_i-1}_{\downarrow} } \max_{w \in \mathcal{B}(H)}  \sum_{x_i = 1}^{k_i-1} w_i(x_i) \rho_i(x_i) 
=
\max_{w \in \mathcal{B}(H)}  \sum_{i=1}^n  \min_{\rho_i \in [0,1]_\downarrow^{k_i-1} } \sum_{x_i = 1}^{k_i-1} w_i(x_i) \rho_i(x_i)  \\
\nonumber 
&  = & 
\max_{w \in \mathcal{B}(H)}  \sum_{i=1}^n  \min_{\rho_i \in [0,1]_\downarrow^{k_i-1} } \sum_{x_i=1}^{k_i-2} \bigg( \sum_{y_i=1}^{x_i} w_i(y_i) \bigg) \big[\rho_i(x_i) - \rho_i(x_i\!+\!1) \big] +
 \bigg( \sum_{y_i=1}^{x_i-1} w_i(y_i) \bigg) \rho_i(k_1 \!-\! 1)\\
\nonumber &  = & 
\max_{w \in \mathcal{B}(H)}  \sum_{i=1}^n  \min_{ \mu_i \in \mathcal{P}(\X_i)  } \sum_{x_i=0}^{k_i-1}  \mu_i(x_i) 
\bigg( \sum_{y_i=1}^{x_i-1} w_i(y_i) \bigg)   \\
\label{eq:dualB} &
= &
\max_{w \in \mathcal{B}(H)} 
\sum_{i=1}^n   \min_{x_i \in \{0,\dots,k_i-1\} } \sum_{y_i=1}^{x_i} w_i(y_i).
\EEA
There is thus two dual problems for the submodular function minimization problem, \eq{dualW} and \eq{dualB}. In practice, checking feasibility in the larger set $\mathcal{W}(H)$ is difficult, while checking feasibility in $\mathcal{B}(H)$ will be done by taking convex combinations of outputs of the greedy algorithm.

\subsection{Strongly-convex separable submodular function minimization}
 \label{sec:sepd}
 
In this section, we consider the separable optimization problem described in \mysec{separable}, now  in the discrete case, where we will be able to use the polyhedra $\mathcal{W}(H)$ and $\mathcal{B}(H)$ defined above.
 We thus consider   functions $a_{i x_i}: \rho_i(x_i) \mapsto a_{i x_i} \big[ \rho_i(x_i) \big]$, for $i \in \{1,\dots,n\}$ and $x_i \in \{1,\dots,k_i-1\}$, which are  convex on $\rb$. For simplicity, we follow the same assumptions than in~\cite[Section 8]{fot_submod}, that is, they are all strictly convex, continuously differentiable and such that the images of their derivatives goes from $-\infty$ to $+\infty$. Their Fenchel conjugates $a_{i x_i}^\ast$ then have full domains and are differentiable.

 We consider the pair of primal/dual problems:
 \BEA
\label{eq:proxd}  & & 
 \min_{\rho} h_\downarrow(\rho) + \sum_{i=1}^n \sum_{x_i=1}^{k_i-1} a_{i x_i}\big[\rho_i(x_i) \big] \mbox{ such that }
 \rho \in \prod_{i=1}^n \rb^{k_i-1}_{\downarrow}  
\\
\nonumber & = & \min_{\rho \in \prod_{i=1}^n\rb^{k_i-1}} \max_{ w \in \mathcal{W}(H) }
\sum_{i=1}^n \sum_{x_i=1}^{k_i-1} \Big\{ w_i(x_i) \rho_i(x_i) + a_{i x_i}\big[\rho_i(x_i) \big]  \Big\}
\mbox{ by Prop.~\ref{prop:WH}}, \\
  &=& \max_{ w \in \mathcal{W}(H) } \sum_{i=1}^n  
\sum_{x_i=1}^{k_i-1} -a_{i x_i}^\ast \big(- w_i(x_i) \big)
.
 \EEA
 Note that we have used $\mathcal{W}(H)$ instead of $\mathcal{B}(H)$ to include automatically the constraints of monotonicity of $\rho$. If using $\mathcal{B}(H)$, we would get a formulation which is more adapted to optimization algorithms (see \mysec{algo}), but not for the theorem below.
 
 The following theorem relates this optimization problem to the following family of submodular minimization problems, for $t \in \rb$:
 \BEQ
 \label{eq:sfmd} \min_{x \in \X} \ H(x) +  \sum_{i=1}^n \sum_{ y_i  = 1}^{x_i}  a_{i y_i}' (t),
 \EEQ
 which is the minimization of the sum of $H$ and a modular function.
 This is a discrete version of Theorem~\ref{theo:sepc}, which directly extends the corresponding results from submodular set-functions.
 \begin{theorem}[Separable optimization -- discrete domains]
 \label{theo:sepdis}
Assume $H$ is submodular. Then:
\BIT
\item[(a)] If $x$ and $x'$ are minimizers of \eq{sfmd} for $t> t'$, then $x \leqslant x'$, i.e., the solutions of \eq{sfmd} are non-increasing in $t \in  \rb$.
\item[(b)] Given a primal candidate $\rho \in \prod_{i=1}^n \rb^{k_i-1}_\downarrow$ and a dual candidate $w \in \mathcal{W}(H)$ for \eq{proxd}, then the duality gap for the problem in \eq{proxd} is the integral from $t=-\infty$ to $t=+\infty$ of the gaps for the problem in \eq{sfmd} for the same dual candidate and the primal candidate $\theta(\rho,t) \in \X$.
\item[(c)] Given the unique solution $\rho$ of \eq{proxd}, for all $t \in \rb$, $\theta(\rho,t) \in \X$ is a solution of \eq{sfmd}. 
\item[(d)] Given solutions $x^t \in \X$ for all problems in \eq{sfmd}, we may define $\rho$ through
$\rho_i(x_i) = \sup \big\{
t \in \rb, \ x_i^t \geqslant x_i
\big\} = \inf \big\{
t \in \rb, \ x_i^t \leqslant x_i
\big\}$, for all $i$ and $x_i$, so that $\rho$ is the optimal solution of \eq{proxd}.
\EIT
\end{theorem}
\begin{proof}
As for Theorem~\ref{theo:sepc},
the first statement (a)  is a consequence of submodularity~\cite{topkis1978minimizing}. The main idea is that when we go from $t$ to $t'<t$, then the function difference  is increasing by convexity, hence the minimizer has to decrease.

For the second statement (b), we provide expressions for all elements of the gap, following the proof in~\cite[Prop.~8.5]{fot_submod}.
 For $M>0$ large enough, we have from \eq{hrho}:
 $$
 h_\downarrow(\rho) - H(0)  = \int_{-M}^{+M} H\big[ \theta(\rho,t) \big] dt - M H(k_1\!-\!1,\dots,k_n\!-\!1) - M H(0).
 $$
 Moreover, the integral of the $i$-th ``non-$H$-dependent'' part of the primal objective   for the submodular minimization problem in \eq{sfmd} is equal to, for $i \in \{1,\dots,n\}$ and the primal candidate $x_i = \theta(\rho_i,t)$:
\BEAS
& & \int_{-M}^M \sum_{y_i=1}^{x_i} a'_{i y_i}(t) dt
\\
& = & \int_{-M}^{+ M} \bigg\{
\sum_{x_i=1}^{k_i-1} \delta( \theta(\rho_i,t)= x_i) \sum_{y_i=1}^{x_i}a_{i y_i}' (t)
\bigg\}
dt \\
& = & \int_{-M}^{\rho_i(k_i-1)}\sum_{y_i=1}^{k_i-1}a_{i y_i}' (t) dt
+ \sum_{s_i=2}^{k_i-1} \int_{\rho_i(s_i)}^{\rho_i(s_i-1)}  \sum_{y_i=1}^{s_i-1}a_{i y_i}' (t)  dt
\mbox{ by definition of } \theta,
\\
& = &   \sum_{y_i=1}^{k_i-1}\big\{  a_{i y_i}( \rho_i(k\!-\!1))  -  a_{i y_i}(-M) \big\} +  \sum_{s_i=2}^{k_i-1} 
 \sum_{y_i=1}^{s_i-1}\bigg\{ a_{i y_i} ( \rho_i(s_i\!-\!1) )- a_{i y_i} ( \rho_i(s_i) ) \bigg\} \\
& = & - \sum_{y_i=1}^{k_i-1} a_{i y_i}(-M) + \sum_{x_i=1}^{k_i-1} a_{i x_i}(\rho_i(x_i)).
\EEAS
We may now compute for any $i \in \{1,\dots,n\}$ and $x_i \in \{0,\dots,k_i-1\}$, the necessary pieces for the  integral of the dual objective values from \eq{dualW}:
\BEAS
 & & \int_{-M}^M \min\{ w_i(x_i) + a'_{i x_i}(t), 0\} dt \\
& =
&  \int_{-M}^{(a_{ix_i}^\ast)'(-w_i(x_i))} \big( w_i(x_i) + a'_{i x_i}(t) \big) dt 
\mbox{ because }  w_i(x_i) + a'_{i x_i}(t) \leqslant 0 \Leftrightarrow t \leqslant (a_{i x_i}^\ast)'(-w_i(x_i)),\\
& = &  a_{i x_i}\big(
(a_{ix_i}^\ast)'(-w_i(x_i))
\big) - a_{i x_i}(-M) + w_i(x_i) \big[ M + (a_{ix_i}^\ast)'(-w_i(x_i))\big] \\
& = & 
  (\psi_{i y_{i}}^\ast)' ( - w_i(x_i))  ( - w_i(x_i)) - 
 \psi_{i y_{i}}^\ast ( - w_i(x_i))- a_{i y_i} (-M)  + w_i(x_i) \big[ M + (a_{ix_i}^\ast)'(-w_i(x_i))\big] \\
 &  = & -a_{i x_{i}}^\ast ( - w_i(x_i)) - a_{i x_i} (-M)  + w_i(x_i)  M.
\EEAS
By putting all pieces together, we obtain the desired result, that is,
$$
 h_\downarrow(\rho) - H(0) + \sum_{x_i=1}^{k_i-1} a_{i x_i}(\rho_i(x_i))
+ \sum_{i=1}^n \sum_{x_i=1}^{k_i-1} a_{i x_{i}}^\ast ( - w_i(x_i))
 $$
 is exactly the integral from $-M$ to $M$ of 
 $$
 H\big[ \theta(\rho,t) \big] + \sum_{i=1}^n \sum_{x_i=1}^{k_i-1} \delta( \theta(\rho_i,t)= x_i) \sum_{y_i=1}^{x_i}a_{i y_i}' (t)
 - H(0) -  \sum_{i=1}^n \min\{ w_i(x_i) + a'_{i x_i}(t), 0\}.
 $$
The proofs of statements (c) and (d) follow the same argument as for 
Theorem~\ref{theo:sepc}.  
\end{proof}
In \mysec{42}, we will use this result for the functions $\varphi_{ix_i}(t) = \frac{1}{2} t^2$, and   from thresholding at $t=0$ we obtain a solution for the minimization of $H$, thus directly extending the submodular set-function situation~\cite{fujishige2005submodular}.

\subsection{Relationship with submodular set-functions and ring families}

For finite sets, it is known that one may reformulate the submodular minimization problem in terms of a submodular function minimization problem with added constraints~\cite{schrijver2000combinatorial}. Given the submodular function $H$, defined on $\X = \prod_{i=1}^n \{0,\dots,k_i-1\}$ we consider a submodular set-function defined on a \emph{ring family} of the set
$$V =\big\{ (i,x_i), \ i \in \{1,\dots,n\}, \ x_i \in \{1,\dots,k_i-1\} \big\}
\subset \{1,\dots,n\} \times \rb
,$$
that is, a family of subsets of $V$ which is invariant by intersection and union. In our context, a member of the ring family is such that  if $(i,x_i)$ is in the set, then all $(i,y_i)$, for $y_i \leqslant x_i$ are in the set as well. The cardinality of $V$ is $\sum_{i=1}^n k_i - n$; moreover, any member of the ring family is characterized for each $i$ by the largest $x_i \geqslant 1$ such that $(i,x_i) \in V$ (if no $(i,x_i)$ is in $V$, then we take $x_i=0$ by convention). This creates a bijection from the ring family to $\X$.

We can identify subsets of $V$ as elements of $\prod_{i=1}^n \{0,1\}^{k_i-1}$. It turns out that elements of the ring family as the non-increasing vectors, i.e., $z \in \prod_{i=1}^n \{0,1\}^{k_i-1}_\downarrow$. Any such element $z$ is also associated uniquely to an element $x \in \X = \prod_{i=1}^n \{0,\dots,k_i-1\}$, by taking $x_i$ as the largest $y_i \in \{1,\dots,k_i \!-\!1\}$ such that $z_i(y_i)=1$, and with value zero, if $z_i=0$. We can thus define a function $H_{\rm ring}(z)$ equal to $H(x)$ for this uniquely associated $x$. Then for any two $z,z' \in \prod_{i=1}^n \{0,1\}^{k_i-1}_\downarrow$ of the ring family with corresponding $x,x' \in \X$, $\min\{z,z'\}$ and $\max\{z,z'\}$ correspond to  $\min\{x,x'\}$ and 
 $\max\{x,x'\}$, which implies that $H_{\rm ring}(z) + H_{\rm ring}(z') \geqslant 
H_{\rm ring}(\max\{z,z'\}) + H_{\rm ring}(\min\{z,z'\})$, i.e., $H_{\rm ring}$ is submodular on the ring family, and minimizing $H(x)$ for $x \in \X$, is equivalent to minimizing $H_{\rm ring}$ on the ring family.

There is then a classical reduction to the minimization of a submodular function on $\prod_{i=1}^n \{0,1\}^{k_i-1}$ (without monotonicity constraints)~\cite{schrijver2000combinatorial,ishikawa2003exact,SchlesingerD06}, which is a regular submodular set-function minimization problem on a set of size $\sum_{i=1}^n k_i-n$ which we now describe (adapted from~\cite[Section 49.3]{schrijver2003combinatorial}). For a certain $B_i > 0$, $i=1,\dots,n$, define the function
$$
H_{\rm ring}^{\rm ext}(z) = H_{\rm ring}(z^\downarrow) + \sum_{i=1}^n B_i \| z_i^\downarrow - z_i\|_1
=    H_{\rm ring}(z^\downarrow) + \sum_{i=1}^n B_i 1_{k_i-1}^\top \big( z_i^\downarrow - z_i \big),
$$
where $1_{k_i-1} \in \rb^{k_i-1}$ is the vector of all ones, and  
where for $z \in \prod_{i=1}^n \{0,1\}^{k_i-1}$, $z^\downarrow$ denotes the smallest element of the ring family containing $z$; in other words $z_i^\downarrow(x_i) = 1$ for all $x_i$ such that there exists $y_i \geqslant x_i$ with $z_i(y_i)=1$, and zero otherwise. We choose $B_i>0$ such that for all $x \leqslant y$, $|H(y)-H(x)|
\leqslant \sum_{i=1}^n B_i \| x_i - y_i\|_1$, so that $H_{\rm ring}^{\rm ext}$ is submodular. Indeed, for any $z,z' \in \prod_{i=1}^n \{0,1\}^{k_i-1}$, we have: 
\BEAS
& & H_{\rm ring}^{\rm ext}(z)  + H_{\rm ring}^{\rm ext}(z')  \\
& = &  H_{\rm ring} (z^\downarrow) + H_{\rm ring} ((z')^\downarrow) + \sum_{i=1}^n B_i \| z_i^\downarrow - z_i\|_1
+ \sum_{i=1}^n \| (z')_i^\downarrow - z'_i\|_1 \\
& \geqslant &  H_{\rm ring} (\max\{ z^\downarrow, (z')^\downarrow \} ) +   H_{\rm ring} (\min\{ z^\downarrow, (z')^\downarrow \} )  + \sum_{i=1}^n B_i \| z_i^\downarrow - z_i\|_1
+ \sum_{i=1}^n \| (z')_i^\downarrow - z'_i\|_1 \\
& & \hspace*{9cm} \mbox{ by submodularity of  } H_{\rm ring}, \\
& = &  H_{\rm ring} (\max\{ z^\downarrow, (z')^\downarrow \} ) +   H_{\rm ring} (\min\{ z^\downarrow, (z')^\downarrow \} )  \\
& & \hspace*{.5cm}+ \sum_{i=1}^n B_i \| \max\{ z_i^\downarrow, (z')_i^\downarrow \} - \max\{ z_i,  z_i'  \}\|_1
+ \sum_{i=1}^n B_i \| \min\{ z_i^\downarrow, (z')_i^\downarrow \} - \min\{ z_i ,  z_i'  \}\|_1 \\
& &  \hspace*{9cm} \mbox{ because } z^\downarrow \geqslant z,\\
& \geqslant &  H_{\rm ring} (\max\{ z, z' \}^\downarrow ) +   H_{\rm ring}  (\min\{ z, z' \}^\downarrow )  \\
& & \hspace*{.5cm} + \sum_{i=1}^n \|  \max\{ z, z' \}_i^\downarrow  - \max\{ z ,  z'  \}_i\|_1
+ \sum_{i=1}^n B_i \| \min\{ z, z' \}_i^\downarrow - \min\{ z ,  z'  \}_i\|_1 \\
& & \hspace*{8cm} \mbox{ because of our choice for } B_i, \\
& = &
 H_{\rm ring}^{\rm ext}(\min\{ z, z' \})  + H_{\rm ring}^{\rm ext}(\max\{ z, z' \}),
\EEAS
which shows submodularity. Moreover, for any strictly positive $B_i$'s,   any minimizer of $H_{\rm ring}^{\rm ext}$ belongs to the ring family, and hence leads to a minimizer of $H$. Therefore, we have reduced the minimization of $H$ to the minimization of a submodular set-function.

The Lov\'asz extension of $H_{\rm ring}^{\rm ext}$ happens to be equal to, for $\nu \in \prod_{i=1}^n [0,1]^{k_i-1}$, 
$$h_\downarrow(\nu^\downarrow) +   \sum_{i=1}^n B_i \sum_{x_i=1}^{k_i-2} \big( \nu_i(x_{i+1}) - \nu_i(x_i) \big)_+,$$
where $\nu_i^\downarrow$ is the smallest non-increasing vector greater or equal to $\nu_i$.
When the $B_i$'s tend to $+\infty$, we recover our convex relaxation from a different point of view. Note that in practice, unless we are in special situations like min-cut/max-flow problems~\cite{ishikawa2003exact} (where we can take $B_i = +\infty$), this strategy adds extra (often unknown) parameters $B_i$ and does not lead to our new interpretations, convex relaxations, and duality certificates, in particular for continuous sets $\X_i$. In particular, when using preconditioned subgradient descent as described in \mysec{41} to solve the submodular set-function minimization problem, for cases where all $B_i$ are equal and all $k_i$ are equal, we get a complexity bound of  $\frac{1}{\sqrt{t}} nk^2 B$ instead of $\frac{1}{\sqrt{t}} nkB$, hence with a worse scaling in the number $k$ of elements of the finite sets, which is due larger Lipschitz-continuity constants.

\section{Optimization for discrete sets}
\label{sec:algorithms}
\label{sec:algo}

In this section, we assume that we are given a submodular function $H$ on $\prod_{i=1}^n \{0,\dots,k_i-1\}$, which we can query through function values (and usually through the greedy algorithm defined in \mysec{greedy}). We assume that for all $x \in \X$, when $x_i < k_i-1$, $|H(x+e_i) - H(x)|$ is bounded by $B_i$ (i.e., Lipschitz-continuity).

We present algorithms to minimize $H$. These can be used in continuous domains by discretizing each $\X_i$. The key is that the overall complexity remains polynomial in $n$, and the dependence in $k_i$ is weak enough to easily reach high precisions.
See the complexity for optimizing functions in continuous domains in the next section.

\subsection{Optimizing on measures}
\label{sec:41}
We have the first equivalent formulations from \mysec{discrete}:
$$
\min_{\mu \in \prod_{i=1}^n \mathcal{P}(\{0,\dots,k_i-1\})} h(\mu)
\ \ \Leftrightarrow  \ \ 
\min_{\rho \in \prod_{i=1}^n [0,1]^{k_i-1}_{\downarrow}} h_\downarrow(\rho) .
$$
Once we get an approximately optimal solution $\rho$, we compute $\theta(\rho, t)$ for all $t \in [0,1]$, and select the one with minimal value (this is a by-product of the greedy algorithm, i.e., with the notation of \mysec{greedy}, $\min_s H\big[ y(s)\big]$). If we have a dual candidate $w \in \mathcal{B}(H)$, we can use \eq{dualB} to obtain a certificate of optimality.

We consider the \emph{projected subgradient method} on $\rho$. We may compute a subgradient of $h_\downarrow$ in $\mathcal{B}(H)$  by the greedy algorithm, and then use $n$ isotonic regressions to perform  the $n$ independent orthogonal projection of $\rho - \gamma h_\downarrow'(\rho)$, using the pool-adjacent-violator algorithm~\cite{best1990active}. Each iteration has thus complexity
$
\Big(\sum_{i=1}^n k_i \Big) \log \Big(\sum_{i=1}^n k_i \Big)
$, which corresponds to the greedy algorithm (which is here the bottleneck).

In terms of convergence rates,  each element 
$w_i(x_i)$ of a subgradient is in $[-B_i,B_i]$. Moreover, for any two elements $\rho, \rho'$ of $\prod_{i=1}^n [0,1]^{k_i-1}_{\downarrow}$, we have $\| \rho_i - \rho'_i \|_\infty \leqslant 1$. Thus, in $\ell_2$-norm, the diameter of the  optimization 
domain is less than $\sqrt{ \sum_{i=1}^n k_i }$, and the norm of subgradients less than
$\sqrt{\sum_{i=1}^n k_i B_i^2}$. Thus, the distance to optimum (measured in function values) after $t$ steps is less 
than~\cite{shor2013nondifferentiable}
$$\frac{1}{\sqrt{t}} \sqrt{    \Big(\sum_{i=1}^n k_i\Big)  \Big(\sum_{i=1}^n k_i B_i^2\Big)}.$$
Note that by using diagonal preconditioning~(see, e.g.,~\cite{juditsky2011first}), we can replace the bound above by
$\frac{1}{\sqrt{t}} \sum_{i=1}^n k_i B_i $. In both cases, if all $B_i$ and $k_i$ are equal, we get a complexity of $O( nk \log(nk))$ per iteration and a convergence rate of
$O( nkB / \sqrt{t} )$. In practice, we choose the Polyak rule for the step-size, since we have candidates for the dual problem; that is, we use $\gamma = \frac{h_\downarrow(\rho) - {D}}{ \|w \|_2^2}$, where $D$ is the best dual value so far, which we can obtain with dual candidates which are the averages of all elements of $\mathcal{B}(H)$ seen so far~\cite{nedic2009}.

Once we have a primal candidate $\rho \in \prod_{i=1}^n [0,1]^{k_i-1}_\downarrow$ and a dual candidate $w \in \mathcal{B}(H)$, we may compute the minimum value of $H$ along the greedy algorithm, as a primal value, and
$$
\min_{\rho \in \prod_{i=1}^n [0,1]^{k_i-1}_\downarrow}\langle w, \rho \rangle
= \sum_{i=1}^n \min_{x_i \in \{0,\dots,k_i-1\} } \sum_{y_i=1}^{x_i} w_i(y_i),
$$
as the dual value. Note that as for submodular set-functions, the only simple way to certify that $w \in \mathcal{B}(H)$ is to make it a convex combination of outputs of the greedy algorithm.

Finally, note that the projected subgradient descent directly applies when a noisy oracle for $H$ is available, using the usual arguments for stochastic extensions~\cite{shor2013nondifferentiable}, with a similar convergence rate.

\paragraph{Minimizing submodular Lipschitz-continuous functions.}
We consider a submodular function $H$ defined on $[0,B]^n$ which is $G$-Lipschitz-continuous with respect to the $\ell_\infty$-norm, that is $|H(x)-H(x')| \leqslant G\| x - x \|_\infty$. By discretizing $[0,B]$ into $k$ values $\frac{ i B}{k}$ for $i \in \{0,\dots,k-1\}$ and minimizing the corresponding submodular discrete function, then we will make an error of at most $\frac{GB}{k}$ on top the bound above, which is equal to $\frac{1}{\sqrt{t}}\sqrt{  (  n k) ( n k B^2 G^2 / k^2 )  } = \frac{BG n}{\sqrt{t}}$ after $t$ iterations of cost  $O(nk \log(nk))$. Thus to reach a global optimization suboptimality of $\varepsilon$, we may take $\frac{GB}{k} = \frac{\varepsilon}{2}$ and $\frac{BG n}{\sqrt{t}}= \frac{\varepsilon}{2}$, that is,
$ k = \frac{2 GB}{\varepsilon}$ and $t = \big(\frac{2 GB n}{\varepsilon}
\big)^2$, leading to an overall complexity of $O\big(
\big(\frac{2 GB n}{\varepsilon} \big)^3 \log \big(\frac{2 GB n}{\varepsilon} \big)
\big)$.

\subsection{Smooth extension and Frank-Wolfe techniques}
\label{sec:42}
We consider the minimization of $h_\downarrow(\rho) + \frac{1}{2} \| \rho \|^2$, with $\| \rho \|^2 = \sum_{i=1}^n \sum_{x_i = 1}^{k_i-1} | \rho_i(x_i) |^2$. Given an approximate $\rho$, we compute $\theta(\rho,t)$ for all $t \in \rb$ (which can be done by ordering all values of $\rho_i(x_i)$, like in the greedy algorithm).

We have by convex duality:
\BEAS
\min_{\rho \in \prod_{i=1}^n \rb^{k_i-1}_\downarrow}
h_\downarrow(\rho) + \frac{1}{2} \| \rho \|^2
& = & \min_{\rho \in \prod_{i=1}^n \rb^{k_i-1}_\downarrow}
\max_{ w \in \mathcal{B}(H)} \langle \rho, w\rangle+ \frac{1}{2} \| \rho \|^2 \\
& = & \max_{ w \in \mathcal{B}(H)}  \bigg\{ \min_{\rho \in \prod_{i=1}^n \rb^{k_i-1}_\downarrow} 
 \langle \rho, w\rangle+ \frac{1}{2} \| \rho \|^2 \bigg\}.
\EEAS 
We thus need to maximize with respect to $w$ in a compact set of diameter less than 
$2 \sqrt{\sum_{i=1}^n k_i B_i^2}$, a $1$-smooth function. Thus, we have that after $t$ steps, the distance to optimum is less
than $ \frac{2}{t} \sum_{i=1}^n k_i B_i^2 $ using Frank-Wolfe techniques, with either line-search of fixed step-sizes~\cite{frank2006algorithm,jaggi,bach2015duality}. Note that primal candidates may also be obtained with a similar convergence rate; the running-time complexity is the same as for subgradient descent in \mysec{41}.

Note that when using a form of line-search, we can use warm-restarts, which could be useful when solving a sequence of related problems, e.g., when discretizing a continuous problem. Finally, more refined versions of Frank-Wolfe algorithms can be used (see \cite{lacoste2015global} and references therein).  

Given an approximate minimizer $\rho$, we can compute all values of $x_i=\theta(\rho_i,t)$ for all $t \in \rb$ from the greedy algorithm to get an approximate minimizer of the original submodular function. Given  an approximate solution with error $\varepsilon$ for the strongly-convex problem, using the exact same argument than for set-functions~\cite[Prop.~10.5]{fot_submod},
we get a bound of $2 \sqrt{\varepsilon \sum_{i=1}^n k_i}$ for the submodular function minimization problem. Combined with the $O(1/t)$ convergence rate described above for the Frank-Wolfe algorithm, there is no gain in the complexity bounds compared to \mysec{41} (note that a similar diagonal pre-conditioning can be used); however the empirical performance is significantly better (see \mysec{experiments}).

\section{Experiments}
\label{sec:experiments}

In this section, we consider a basic experiment to illustrate our results. We consider the function defined on $[-1,1]^n$,
$$H(x) = \frac{1}{2} \sum_{i=1}^n ( x_i - z_i)^2 + \lambda \sum_{i=1}^n |x_i|^\alpha + \mu \sum_{i=1}^{n-1} ( x_i - x_{i+1})^2.$$
This corresponds to denoising a one-dimensional signal $z$ with the contraint that $z$ is smooth and sparse. See an illustration in \myfig{signals}.
The smoothness prior is obtained from the quadratic form (which is submodular and convex), while the sparse prior is only submodular, and not convex. Thus, this is not a convex optimization problem when $\alpha<1$; however, it can be solved globally to arbitrary precision for any $\alpha>0$ using the algorithms presented in \mysec{algorithms}. Note that the use of a non-convex sparse prior (i.e., $\alpha$ significantly less than one, e.g., $1/8$ in our experiments) leads to fewer biasing artefacts than the usual $\ell_1$-norm~\cite{fan2001variable}.

In \myfig{gaps}, we show certified duality gaps for the two algorithms from \mysec{algorithms} on a discretization with $50$ grid elements for each of the $n=50$ variables. We can see that the Frank-Wolfe-based optimization performs better than projected subgradient descent, in particular the ``pairwise-Frank-Wolfe'' method of~\cite{lacoste2015global}.

Finally, in \myfig{rhos}, we show the estimated values for the vectors $\rho$, showing that the solution is almost a threshold function for the non-smooth dual problem used in \mysec{41} (left), while it provides more information in the smooth dual case  used in \mysec{42} (right), as it solved a series of submodular function minimization problems.

\section{Discussion}

In this paper, we have shown that a large family of non-convex problems can be solved by a well-defined convex relaxation and efficient algorithms based on simple oracles. This is based on replacing each of the variables $x_i$ by a probability measure and optimizing over the cumulative distributions of these measures. Our algorithms apply to all submodular functions defined on products of subsets of $\rb^n$, and hence include continuous domains as well as finite domains. 
This thus defines a new class of functions that can be minimized in polynomial time.

\begin{figure} 
\begin{center}
\includegraphics[scale=.5]{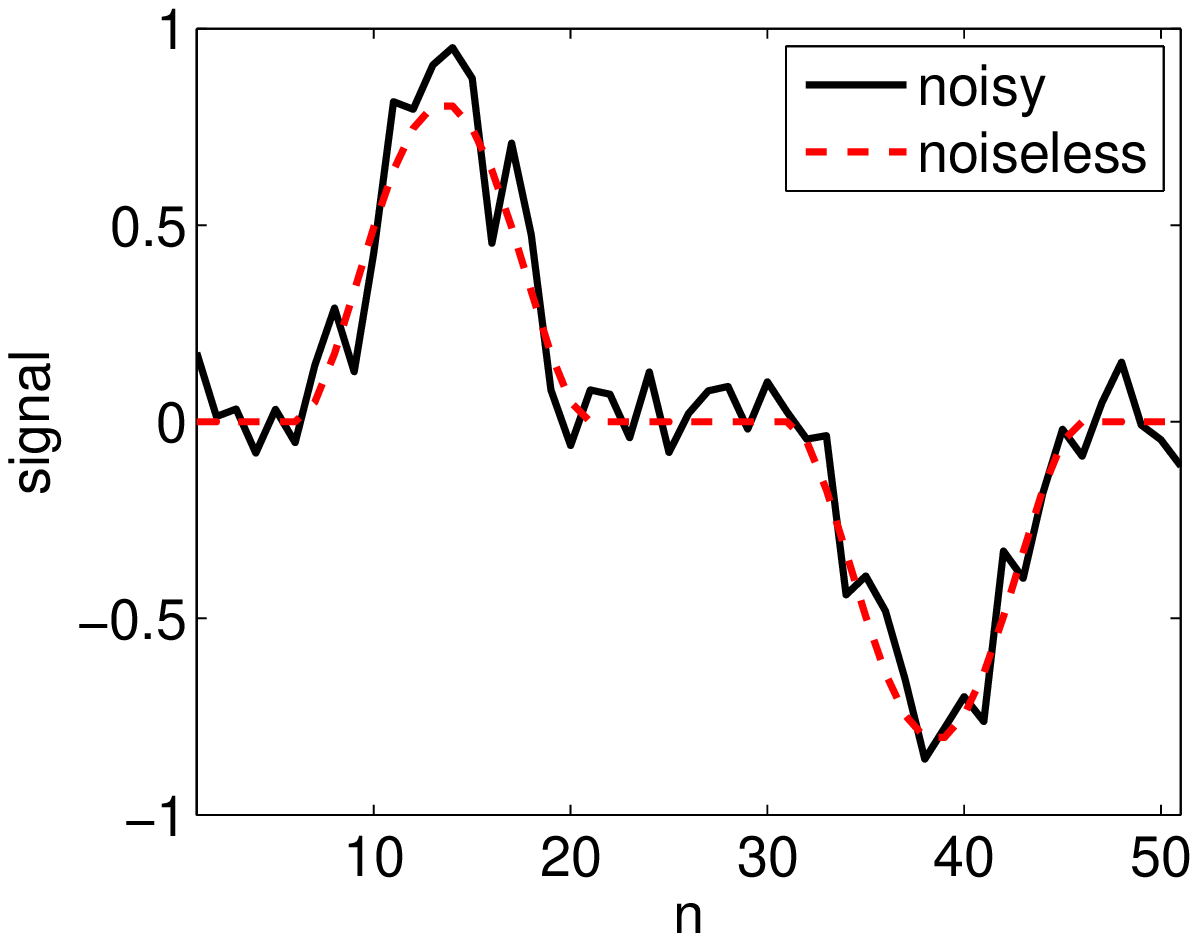} \hspace*{1cm}
\includegraphics[scale=.5]{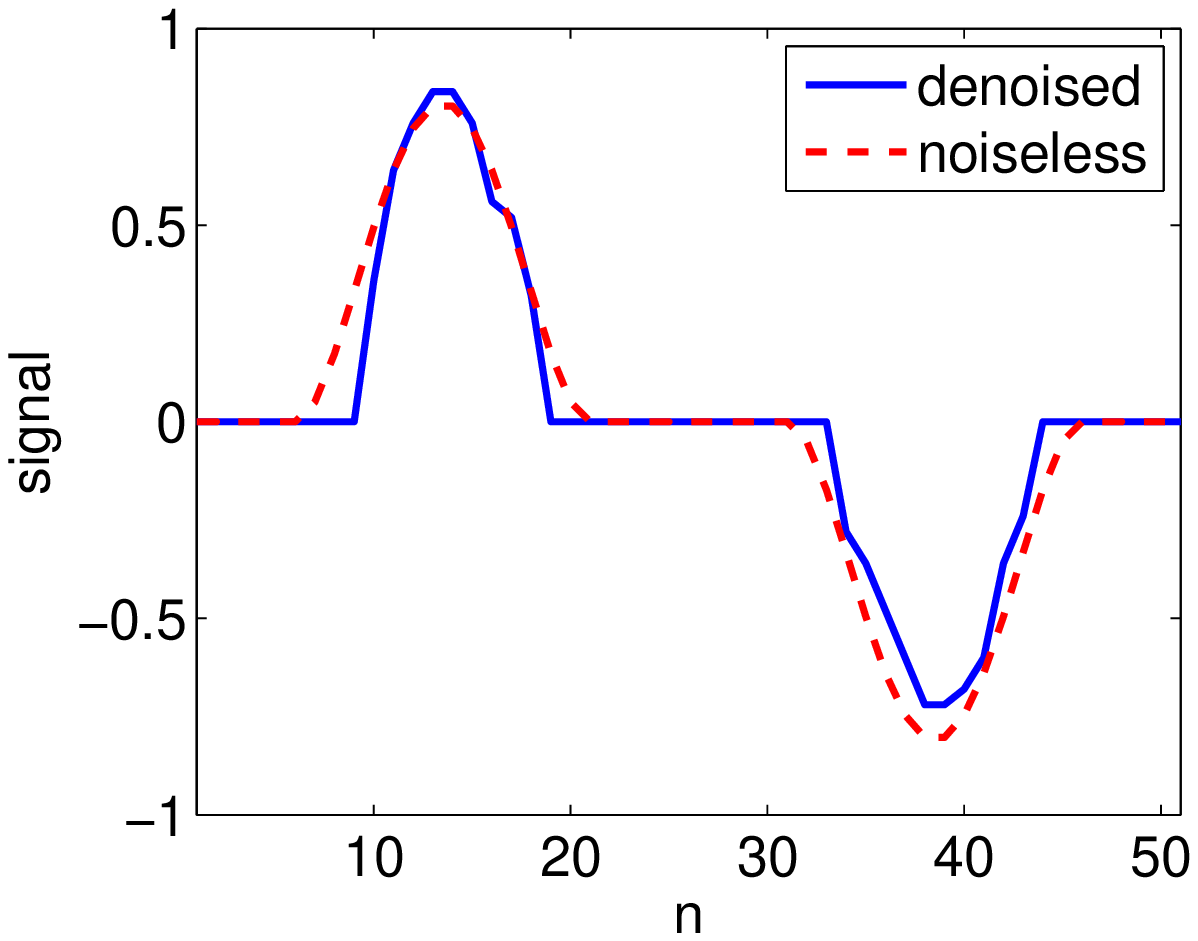}
\end{center}

\vspace*{-.5cm}

\caption{One-dimensional signals: noisy input obtained by adding Gaussian noise to a noiseless signal (left); denoised signal (right).}
\label{fig:signals}
\end{figure}

\begin{figure} 
\begin{center}
\includegraphics[scale=.5]{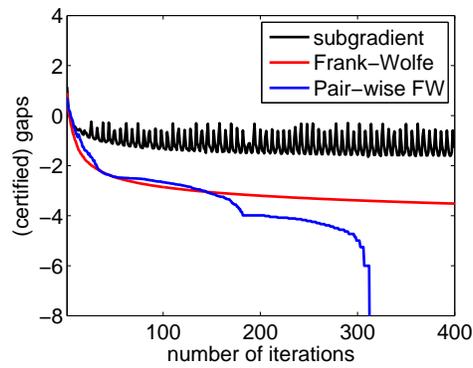}
\end{center}

\vspace*{-.5cm}

\caption{Certified duality gaps: noisy input obtained by adding Gaussian noise to a noiseless signal (left); denoised signal (right).}
\label{fig:gaps}
\end{figure}

\begin{figure} 
\begin{center}

\includegraphics[scale=.5]{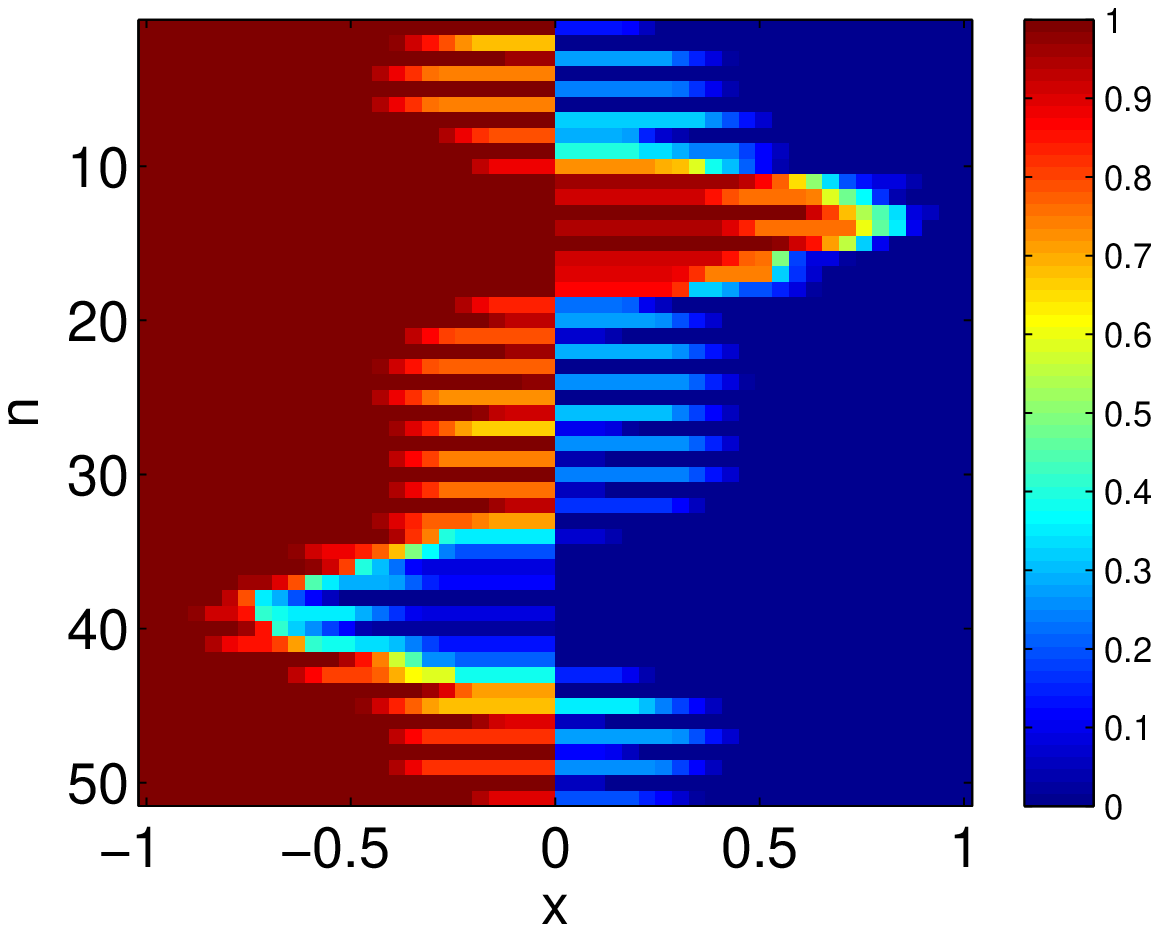}  \hspace*{1cm}
\includegraphics[scale=.5]{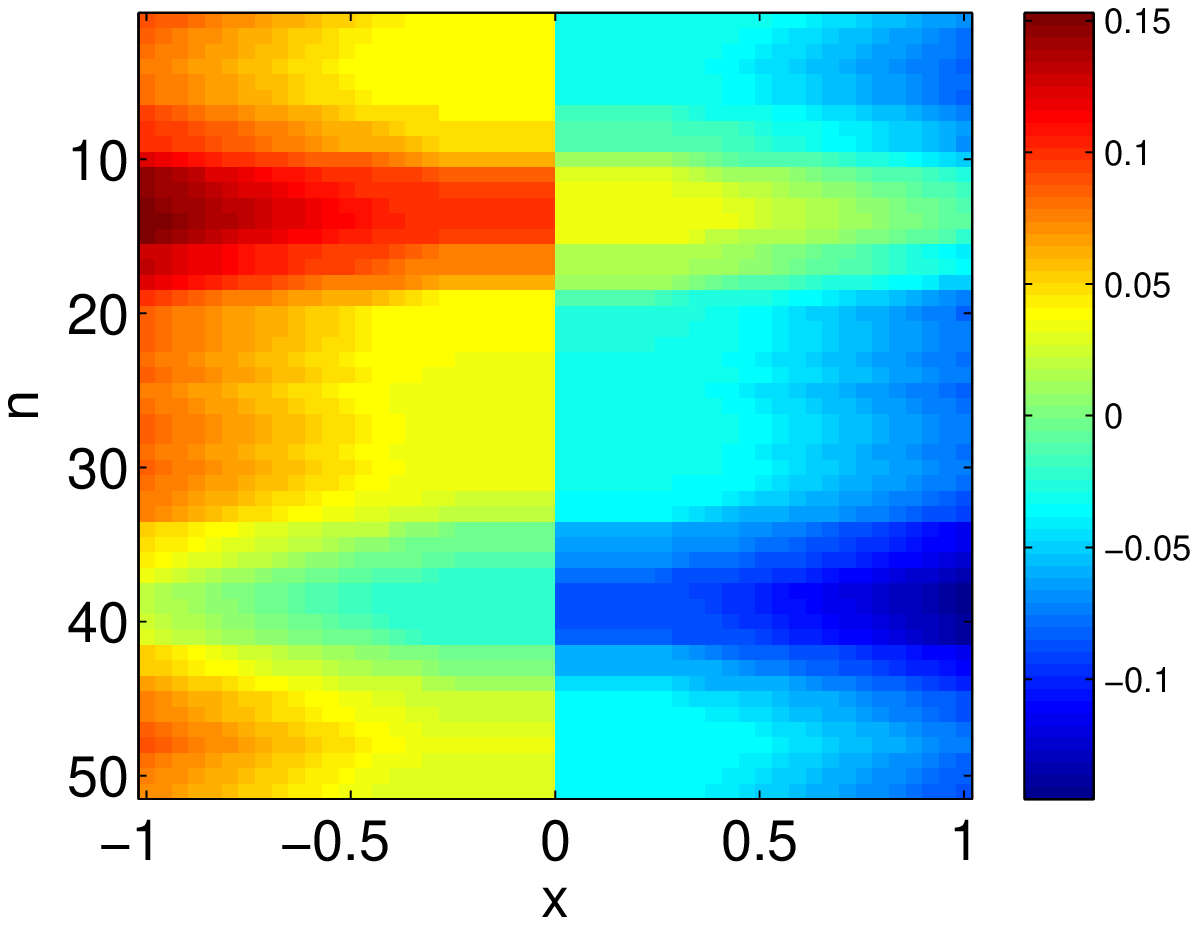}
\end{center}

\vspace*{-.5cm}

\caption{Optimal values of $\rho$: non-smooth dual problem (left); smooth dual problem (right).  }
\label{fig:rhos}
\end{figure}

Several extensions are worth considering:

\BIT
\item[--] \textbf{Relationship with convexity for continuous domains}: for functions defined on a product of sub-intervals, there are two notions of ``simple'' functions, convex and submodular. These two notions are usually disjoint (see \mysec{examples} for examples). We study two interesting relationships: (a) the convex closure we define in \mysec{closure} for convex functions, and (b) the minimization of the sum of a submodular function and a convex function.

Given a convex function $G$ defined on a product $\X$ of intervals, its convex closure defined on $\mathcal{P}^\otimes (\X)$ is defined as:
$$g_{\rm closure}(\mu) = \inf_{ \gamma \in \Pi(\mu) } \int_\X G(x) d \gamma(x).$$
Given Jensen's inequality, it satisfies:
$$
g_{\rm closure}(\mu) \geqslant   \inf_{ \gamma \in \Pi(\mu) }  G\bigg( \int_\X x  d \gamma(x) \bigg)
= G\bigg( \int_\X x  \prod_{i=1}^n d\mu_i(x_i) \bigg),
$$
that is one may obtain a lower bound on the convex closure (note that if perform the convex closure directly on $\X$ and not on $\mathcal{P}\otimes(\X)$, we would obtain $G$). This lower bound on the convex closure may be tight in some situations, for example, when minimizing a   convex function  $G  $ from $\X$ to $\rb$. Indeed, we have:
\BEAS
\inf_{x \in \X}   G(x) 
& = &  \inf_{ \mu \in \mathcal{P}^\otimes(\X) }G_{\rm closure}(\mu) \\
& \geqslant &  \inf_{ \mu \in \mathcal{P}^\otimes(\X) } G\bigg( \int_\X x  \prod_{i=1}^n d\mu_i(x_i) \bigg) = \inf_{x \in \X}   G(x),
\EEAS
and thus the convex closure and its relaxation allow an exact minimization, which allows to extend the result of~\cite{wald2014tightness,ruozzi} from graphical model-based functions to all convex functions.

Moreover, it is worth considering the two notions of submodularity and convexity \emph{simultaneously}, as many common objective functions are the sums of a convex function and a submodular function (e.g., the negative log-likelihood of a Gaussian vector, where the covariance matrix is parameterized as a linear combination of positive definite matrices~\cite{murphy2012machine}). We can thus  consider the minimization of $H(x) + G(x)$, where $H: \X \to \rb$ is submodular and $G: \X \to \rb$ is convex. From the same reasoning as above, we have a natural convex relaxation on our set of measures:
\BEQ
\label{eq:relaxcvx}
\min_{ \mu \in \mathcal{P}^\otimes(\X)} h(\mu) + G \bigg(
\int_\X x \prod_{i=1}^n d\mu_i(x_i)
\bigg).
\EEQ
Another relaxation is to replace $H$ by its convex envelope on $\X$ (which is computable, as one can miminizer $H$ plus linear functions), which we can get by bi-conjugation. We have for $z \in \rb^n$:
\BEAS
H^\ast(z) & = & \sup_{ x \in \X} x^\top z - H(x) = \sup_{ \mu \in \mathcal{P}^\otimes(\X)} z^\top \int_\X x \prod_{i=1}^n d\mu_i(x_i) - h(\mu).
\EEAS
This implies that for any $x \in \X$:
\BEAS
H^{\ast \ast}(x) & = & \sup_{ z \in \rb^n} \inf_{ \mu \in \mathcal{P}^\otimes(\X)}
z^\top x - z^\top \int_\X x \prod_{i=1}^n d\mu_i(x_i) + h(\mu) \\
& = & \inf_{ \mu \in \mathcal{P}^\otimes(\X)}  \sup_{ z \in \rb^n} z^\top \bigg(
x - \int_\X x \prod_{i=1}^n d\mu_i(x_i)  \bigg)
 + h(\mu)\\
& = & \inf_{\mu \in \mathcal{P}^\otimes(\X)} h(\mu) \mbox{ such that } \int_\X y \prod_{i=1}^n d\mu_i(y_i) = x.
\EEAS
This implies that the relaxation in \eq{relaxcvx} is equivalent to the minimization of $H^{\ast\ast}(x)+ G(x)$, which is another natural convex relaxation.
The main added benefit of submodularity is that the convex envelope can be computed when~$H$ is submodular, whereas typically, it is not possible.

\item[--] \textbf{Submodular relaxations}: it is possible to  write most functions as the difference of two submodular functions $H$ and $G$, leading to an exact reformulation in terms of minimizing $h(\mu) - g(\mu)$ with respect to the product of measures $\mu$. This problem is however  non-convex anymore in general, and we could use majorization-minimization procedures~\cite{lange2000optimization}. Alternatively, if a function is a sum of simple functions, we may consider ``submodular relaxations'' of each of the component (a situation comparable with replacing functions by their convex envelopes). Like in the set-function case, a notion of submodular envelope similar to the convex envelope is not available. However, for any function $H$, one can always define a convex extension $\tilde{h}(\mu)$ through an optimal transport problem as $ {h}_{\rm closure} (\mu) = \inf_{ \gamma \in \Pi(\mu) } \int_{\X} H(x) d\gamma(x)$. For functions of two variables, this can often be computed in closed form, and by taking the sum of these relaxations, we exactly get the usual linear programming relaxation (see~\cite{zwp14:mit} and references therein).

\item[--] \textbf{Submodular function maximization}: While this paper has focused primarily on minimization, it is worth exploring if algorithms for the maximization of submodular set-functions can be extended to the general case~\cite{nemhauser1978analysis,feige2007maximizing}, to obtain theoretical guarantees.

\item[--] \textbf{Divide-and-conquer algorithm}: For submodular set-functions, the separable optimization problem defined in \mysec{separable} can be exactly solved by a sequence of at most $n$ submodular optimization problem by a divide-and-conquer procedure~\cite{groenevelt1991two}. It turns out that a similar procedure extends to general submodular functions on discrete domains (see Appendix~\ref{app:dc}).

\item[--] \textbf{Minimizing sums of simple submodular functions}: Many submodular functions turn out to be decomposable as the sum of ``simple'' functions, that is functions for which minimization is particularly simple (see~\cite{komodakis2011mrf} for examples from computer vision). For submodular set-functions, decomposability has been used to derive efficient combinatorial~\cite{Kolmogorov-DAM-2012} or convex-optimization-based~\cite{jegelka2013reflection} algorithms. These could probably be extended to  general submodular functions.

\item[--] \textbf{Active-set methods}: For submodular set-functions, active-set techniques such as the minimum-norm-point algorithm have the potential to find the exact minimizer in finitely many iterations~\cite{wolfe1976finding,fujishige2005submodular}; they are based on the separable optimization problem from \mysec{separable}. Given that our extension simply adds inequality constraints, such an approach could easily be extended.

\item[--] \textbf{Adaptive discretization schemes}: Faced with functions defined on continuous domains, currently the only strategy is to discretize the domains; it would be interesting to study adaptive discretization strategies based on duality gap criteria.

\EIT

 \subsection*{Acknowledgements}
 
 The author thanks Marco Cuturi, Stefanie Jegelka, Gabriel Peyr\'e, Suvrit Sra, and Tomas Werner, for helpful discussions and feedback.  The main part of this work was carried while visiting the Institut des Hautes Etudes Scientifiques (IHES) in Bures-sur-Yvette, supported by the Schlumberger Chair for mathematical sciences.

\bibliography{transport}

\appendix

\section{Proof of miscellaneous results}

\subsection{Submodularity of the Lov\'asz extension of a submodular set-function}

\label{app:subsub}

We consider a submodular set-function $G: \{0,1\}^n \to \rb$ and its Lov\'asz extension $g: \rb^n \to \rb$. In order to show the submodularity of $g$, we simply apply the definition and consider $x \in \rb^n$ and two distinct basis vectors $e_i$ and $e_j$ with infinitesimal positive displacements $a_i$ and $a_j$. If $(i,j)$ belongs to two different level sets of $x$, then, for $a_i$ and $a_j$ small enough,   $g(x+a_i e_i) + g(x+a_j e_j) - g(x)  - g(x+a_i e_i + a_j e_j)$ is equal to zero. If $(i,j)$ belongs to the same level sets, then, by explicitly computing the quantity for $a_i>a_j$ and $a_i<a_j$, it is non-negative (as a consequence of submodularity).

Note that then, the extension has another expression when $\X = [0,1]^n$, as, if $H = g$,
\BEAS
h(\mu) & = & \int_0^1 g( F_{\mu_1}^{-1}(t),\dots,F_{\mu_n}^{-1}(t)) dt \\
& = & \int_0^1 \int_0^1  G(  \{ i, F_{\mu_i}^{-1}(t) \geqslant z \} ) dt  dz
= \int_0^1 \int_0^1  G(  \{ i, F_{\mu_i}(z) \geqslant t \} ) dt  dz \\
&  = &  \int_0^1 g( F_{\mu_1}(z),\dots,F_{\mu_n}(z)) dz,
\EEAS
which provides another proof of submodularity since $g$ is convex.

\subsection{Submodularity of the multi-linear extension of a submodular set-function}
We consider a submodular set-function $G: \{0,1\}^n \to \rb$ and its multi-linear extension~\cite{vondrak2008optimal}, defined as a function $\tilde{g}:[0,1]^n \to \rb$ with
$$\tilde{g}(x_1,\dots,x_n)  = \E_{y_i\sim {\rm Bernoulli}(x_i), i \in \{1,\dots,n\}} G(y),$$
where all Bernoulli random variables $y_i$ are independent. In order to show submodularity, we only need to consider the case $n=2$, for which the function $\tilde{g}$ is quadratic in $x$ and the cross-term is non-positive because of the submodularity of $G$. See more details in~\cite{vondrak2008optimal}. The extension on a product of measures on $[0,1]^n$ does not seem to have the same simple interpretation as for the Lov\'asz extension in Appendix~\ref{app:subsub}.

\label{app:subsubmulti}

\section{Divide-and-conquer algorithm for separable optimization}
\label{app:dc}

We consider the optimization problem studied in \mysec{sepd}:
$$
 \min_{\rho} h_\downarrow(\rho) + \sum_{i=1}^n \sum_{x_i=1}^{k_i-1} a_{i x_i}\big[\rho_i(x_i) \big] \mbox{ such that }
 \rho \in \prod_{i=1}^n \rb^{k_i-1}_{\downarrow} ,
 $$
 whose dual problem is the one of  maximizing $  \sum_{i=1}^n \sum_{x_i=1}^{k_i-1} -a_{i x_i}^\ast(-w_i(x_i))$ such that $w \in \mathcal{W}(H)$. We assume that all functions $a_{ix_i}$ are strictly convex and differentiable with a Fenchel-conjugate with full domain.
 From Theorem~\ref{theo:sepdis}, we know that it is equivalent to a sequence of submodular minimization problems of the form:
$$
\min_{x \in \X} H(x) +  \sum_{i=1}^n \sum_{ y_i  = 1}^{x_i}  a_{i y_i}' (t),
$$
i.e., minimizing $H$ plus a modular function.
We now show that by solving a sequence of such problems (with added restrictions on the domains of the variables), we recover exactly the solution $\rho$ of the original problem. This algorithm directly extends the one from~\cite{groenevelt1991two}, and we follow the exposition from~\cite[Section 9.1]{fot_submod}.  The recursive algorithm is as follows:
\BIT
\item[(1)] Find the unique global maximizer of $\displaystyle - \sum_{i=1}^n \sum_{x_i=1}^{k_i-1} a_{ix_i}^\ast(-w_i(x_i))$
with the single constraint that
$\displaystyle\sum_{i=1}^n \sum_{x_i=1}^{k_i-1} w_i(x_i) = H(k_1,\dots,k_n) - H(0)$. This can be typically obtained in closed form, or through the following one-dimensional problem (obtained by convex duality),
$\displaystyle 
\min_{ t \in \rb} \big[ H(k_1-1,\dots,k_n-1) - H(0) \big] t  + \sum_{i=1}^n \sum_{x_i=1}^{k_i-1} a_{ix_i}(t)
$, with $w_i(x_i)$ then obtained as $w_i(x_i) = - a_{i x_i}'(t)$.

\item[(2)] Find an element $y \in \X$ that minimizes $\displaystyle H(y) - \sum_{i=1}^n \sum_{z_i=1}^{y_i} w_i(z_i)$. This is a submodular function minimization problem.

\item[(3)] If $\displaystyle H(y) - \sum_{i=1}^n \sum_{z_i=1}^{y_i} w_i(z_i) = H(0)$, then exit ($w$ is optimal).
\item[(4)] Maximize $  \sum_{i=1}^n \sum_{x_i=1}^{y_i} -a_{i x_i}^\ast(-w_i(x_i))$ over $w \in \mathcal{W}(H_{x \leqslant y })$, where $H_{x \leqslant y}$ is the restriction of $H$ to $\{x \leqslant y\}$, to obtain $w^{\{x \leqslant y\}}$.
\item[(5)] Maximize $  \sum_{i=1}^n \sum_{x_i=y_i+1}^{k_i-1} -a_{i x_i}^\ast(-w_i(x_i))$ over $w \in \mathcal{W}(H_{x \geqslant y+1 })$, where $H_{x \geqslant y+1}$ is the restriction of $H$ to $\{x \geqslant y+1\}$, to obtain $w^{\{x \geqslant y+1\}}$. 
\item[(6)] Concatenate the two vectors $w^{\{x \geqslant y+1\}}$ and $w^{\{x \leqslant y\}}$ into $w$, which is the optimal solution.
\EIT
 
The proof of correctness is the same as for set-functions~\cite[Section 9.1]{fot_submod}: if the algorithm stops at step (3), then we indeed have the optimal solution because the optimum on a wider set happens to be in $\mathcal{W}(H)$. 
Since  the optimal $w$ in (1) is such that $w_i(x_i) = - a_{i x_i}'(t)$ for all $i$ and $x_i$ and a single real number $t$, the problem solved in (2) corresponds to one of the submodular function minimization problems from Theorem~\ref{theo:sepdis}.
From the statement (d) of that theorem, we know that the optimal primal solution $\rho$ will be such that $\rho_i(x_i) \geqslant t$ for $x_i \leqslant y_i$ and $\rho_i(x_i) \leqslant t$ for $x_i \geqslant y_i+1$. Given the expression of   $h_\downarrow(\rho)$ obtained from the greedy algorithm, if we impose that $\min_{ x_i \leqslant y_i} \rho_i(x_i) \geqslant  \max_{ x_i \geqslant y_i+1} \rho_i(x_i) $, the function $h_\downarrow(\rho)$ is the sum of two independent terms and the minimization of the primal problem can be done into two separate pieces, which correspond exactly to $h_{x \leqslant y }^\downarrow(\rho^{x \leqslant y })$
and $h_{x \geqslant y+1 }^\downarrow( \rho^{x \geqslant y+1 })$. If we minimize the two problems separately, like done in steps (4) and (5), we thus only need to check that the decoupled solutions indeed have the correct ordering, that is
$\min_{ x_i \leqslant y_i} \rho_i(x_i) \geqslant  \max_{ x_i \geqslant y_i+1} \rho_i(x_i) $, which is true because of 
Theorem~\ref{theo:sepdis} applied to the two decoupled problems with the same value of $t$.

\end{document}